\documentclass{article}

\usepackage{iclr2027/fancyhdr,iclr2027/natbib}
\usepackage{iclr2027/iclr2027_conference,times}
\iclrfinalcopy

\usepackage[utf8]{inputenc} 
\usepackage[T1]{fontenc}    
\usepackage[hidelinks]{hyperref} 
\usepackage{url}            
\usepackage{booktabs}       
\usepackage{amsfonts,amssymb,amsmath,amsthm}
\usepackage{algorithm}
\usepackage{algpseudocode}
\usepackage{mathtools}
\usepackage{nicefrac}       
\usepackage{microtype}      
\usepackage{xcolor}         
\usepackage{graphicx}       
\usepackage{subcaption}     

\newtheorem{lemma}{Lemma}
\newtheorem{proposition}{Proposition}
\newtheorem{observation}{Observation}

\newcommand{\E}{\mathbb{E}}
\newcommand{\R}{\mathbb{R}}
\newcommand{\Cov}{\mathrm{Cov}}
\newcommand{\sgn}{\mathrm{sgn}}

\definecolor{exactcolor}{HTML}{0072B2} 
\definecolor{gprcolor}{HTML}{E69F00}   
\newcommand{\dotexact}{\textcolor{exactcolor}{\rule[0.25ex]{0.6em}{0.6em}}}
\newcommand{\dotgpr}{\textcolor{gprcolor}{\rule[0.25ex]{0.6em}{0.6em}}}
\newcommand{\dotjbb}{\textcolor{gprcolor}{\rule[0.25ex]{0.6em}{0.6em}}}
\newcommand{\dothb}{\textcolor{exactcolor}{\rule[0.25ex]{0.6em}{0.6em}}}

\usepackage{amssymb}
\usepackage{xcolor}

\newcommand{\mkFGSM}{\textcolor{gray}{\scalebox{0.9}{$\bullet$}}}
\newcommand{\mkFGM}{\textcolor{gray}{\scalebox{0.7}{$\blacksquare$}}}
\newcommand{\mkRSFGSM}{\textcolor{gray}{\scalebox{0.8}{$\blacktriangle$}}}
\newcommand{\mkPGDtwo}{\textcolor{gray}{\scalebox{0.8}{$\blacklozenge$}}}
\newcommand{\mkPGDfive}{\textcolor{gray}{\scalebox{0.8}{$\blacktriangledown$}}}
\newcommand{\mkPGDten}{\textcolor{gray}{\scalebox{0.9}{$\boldsymbol{+}$}}}

\title{Fast Adversarial Attacks\\with Gradient Prediction}
\colorlet{sashafixcolor}{red}

\newcommand{\JBBQwenThreeN}{{50}}

\newcommand{\JBBQwenTwoN}{{50}}
\newcommand{\JBBQwenTwoFGSMExactASR}{{22.0}}
\newcommand{\JBBQwenTwoFGSMGPRASR}{{4.0}}
\newcommand{\JBBQwenTwoFGMExactASR}{{8.0}}
\newcommand{\JBBQwenTwoFGMGPRASR}{{22.0}}
\newcommand{\JBBLlamaTwoN}{{50}}

\newcommand{\AlphaBootstrapB}{{10{,}000}}
\newcommand{\AlphaHBQwenTwoN}{{294}}

\newcommand{\AlphaHBLlamaTwoN}{{280}}

\newcommand{\AlphaJBBLlamaTwoZeroASR}{{0.0}}
\newcommand{\AlphaJBBLlamaTwoSmallASR}{{6.0}}
\newcommand{\AlphaJBBLlamaTwoOneASR}{{6.0}}

\author{%
  \normalfont
  \begin{tabular}[t]{@{}p{0.46\textwidth}@{\hspace{0.04\textwidth}}p{0.46\textwidth}@{}}
    \textbf{Kamil Ciosek} & \textbf{Aleksandr V. Petrov} \\
    Spotify & Spotify \\
    \texttt{kamilc@spotify.com} & \texttt{aleksandrv@spotify.com} \\[1.5em]
    \textbf{Nicol\`{o} Felicioni} & \textbf{Konstantina Palla} \\
    Spotify & Spotify \\
    \texttt{nicolof@spotify.com} & \texttt{konstantinap@spotify.com}
  \end{tabular}%
}

\begin{document}
\maketitle
\lhead{Preprint}

\begin{abstract}
Generating adversarial examples at scale is a core primitive for robustness evaluation, adversarial training, and red-teaming, yet even ``fast'' attacks such as FGSM remain throughput-limited by the cost of a backward pass. We introduce a family of attacks that eliminates the backward pass by predicting the input gradient from forward-pass hidden states via a lightweight linear regression. Theoretically, we derive exact affine conditional gradient means, showing optimality in the (idealized) NTK regime. Empirically, our methods work when applied to practical finite-width models; we recover much of FGSM's attack performance while using only a small fraction of the time, corresponding to a $532\%$ increase in throughput. These results suggest gradient prediction as a simple and general route to significantly faster adversarial generation under realistic wall-clock constraints.
\end{abstract}

\section{Introduction}
Adversarial examples expose small, structured perturbations that can reliably change a model's prediction \citep{szegedyIntriguing2014}. Beyond their diagnostic value, adversarial inputs are a practical tool for robustness evaluation and for training procedures that explicitly optimize against worst-case perturbations \citep{goodfellowAdversarial2015,madry2018}. A recurring operational obstacle is \emph{throughput}: modern robustness pipelines are often constrained by wall-clock time and accelerator availability, so the number of adversarial examples produced per unit time can matter as much as attack success rate.

Even ``one-step'' attacks inherit a fundamental bottleneck. FGSM uses a single update but still requires an input gradient of a chosen attack score \citep{goodfellowAdversarial2015}. For large networks (and especially large language models), this typically means an additional backward pass per example, which can dominate the cost of adversarial data generation.

We address this limitation by eliminating the backward pass. Our key idea is to \emph{predict} the input gradient from quantities already produced during the forward pass: a hidden state (representation) vector. Concretely, we fit a lightweight linear regressor that maps hidden states to input gradients and then perform an FGSM-style update using the predicted gradient. This design is motivated by a kernel/NTK view of neural networks \citep{jacot2018neural,lee2018deep,matthews2018gaussian,yang2019wide}, and is closely related in spirit to synthetic gradient ideas \citep{jaderberg2017decoupled,czarnecki2017understanding} (here applied to \emph{input} gradients rather than activation gradients).


\paragraph{Contributions.} We propose a family of methods that replaces attack-time backward passes with a trained affine gradient predictor. Our objective is successful attacks per unit of generation time, with ASR as a complementary metric. For FGSM, successful-attack throughput increases by 532\% at the cost of a drop in ASR (Table~\ref{tab:main}). Therefore GPR occupies part of the ASR-throughput Pareto frontier (Figure~\ref{fig:pareto}). 

\section{Preliminaries}

\paragraph{Adversarial attacks.}
We consider a fixed trained network $f_\theta$. Let $x\in\R^d$ denote the (continuous) input we perturb; for language models, $x$ can be understood as a flattened input embedding tensor. Let $z_\theta(x)\in\R^C$ denote the model's logits on $x$, and let $h_\theta(x)\in\R^{d_h}$ denote the hidden state.

Given a target output $t$, define a scalar score
\[
S_\theta(x,t)\in\R,
\]
which measures support for $t$. A logit-based classification score is $S_\theta(x,t)=z_{\theta,t}(x)$. The score is an implementation detail: we use a logit-based score for embedding attacks and a loss-based score for token-level attacks. We assume differentiability with respect to the continuous input and define the input gradient
\[
g_\theta(x,t)\;\coloneqq\;\nabla_x S_\theta(x,t)\in\R^d.
\]
A targeted FGSM-style perturbation takes a single step in the sign of the target-score gradient:
\[
x_{\mathrm{adv}}\;=\;x+\varepsilon\,\sgn\!\big(g_\theta(x,t)\big).
\]

\paragraph{Gaussian processes (GPs) and the neural tangent kernel (NTK).}
A (vector-valued) Gaussian process (GP) is a collection of random vectors
$\{f(x)\in\R^m:x\in\R^d\}$ such that for any finite set of inputs $\{x_1,\dots,x_n\}$, the stacked vector
$[f(x_1)^\top,\dots,f(x_n)^\top]^\top$ is multivariate Gaussian \citep{williams2006gaussian}. A GP is determined by its mean function
and its kernel (covariance) function. Wide neural networks with i.i.d.\ Gaussian weights converge, at random initialization (and with a suitable parametrization), to a GP: the joint distribution of outputs becomes (approximately) Gaussian as widths grow. As widths $\to\infty$ (with depth fixed and appropriate scaling), the network converges in finite-dimensional distribution to a GP with a deterministic kernel. This correspondence extends to broad architecture classes beyond fully-connected networks \citep{yang2019wide}. The NTK perspective \citep{jacot2018neural} studies gradient-based training in the same wide-network scaling. In this regime, training dynamics under (suitably scaled) gradient descent are described by a kernel method with the limiting NTK computed at initialization \citep{jacot2018neural} and nearly constant during training. The limiting final outputs are Gaussian at initialization and remain Gaussian. Moreover, with suitable input differentiability, their input derivatives are jointly Gaussian with them (Appendix~\ref{app:gp-derivatives}), a fact our method crucially relies on.

\paragraph{Gaussian conditioning for gradients.}
Denote by $H(x)$ the vector random variable representing the hidden state and by $G(x)$ the network input gradient. If $(H,G)$ is jointly Gaussian with means $\mu_H,\mu_G$ and covariances
$\Sigma_{HH}=\Cov(H,H)$ and $\Sigma_{GH}=\Cov(G,H)$, with $\Sigma_{HH}$ invertible, then the conditional mean is affine:

\begin{equation}
\E\!\left[G\mid H\right]
\;=\;
\mu_G+\Sigma_{GH}\Sigma_{HH}^{-1}(H-\mu_H).
\label{eq:gauss-cond}
\end{equation}
For the NTK specialization, take $H(x)$ to be the final hidden state and $G(x)=\nabla_x(c^\top H(x))$ for a fixed vector $c$. Under NTK training, we get joint Gaussianity of $(H(x),G(x))$; this follows from the GP derivative rule rather than being an extra Gaussianity assumption. Note that nonlinear activation functions inside the network are allowed. Appendix~\ref{app:gp-derivatives} gives more details.

\paragraph{Stationary kernels.}
A kernel $k:\R^d\times\R^d\to\R$ is \emph{stationary} if $k(x+c,x'+c)=k(x,x')$ for all $x,x',c$,
equivalently $k(x,x')=\psi(x-x')$ for some $\psi$. In our idealized analysis, joint stationarity of $(H,Y)$, together with the stated moment and differentiability conditions, makes the relevant prior means and same-input covariances of $(H(x),\nabla_xY(x))$ input-independent.\footnote{See Section~\ref{sec-method} for proof.}

\paragraph{Embedding vs token-level attacks (text models).}
Embedding-level attacks optimize continuous input embeddings directly, while token-level attacks operate in the discrete space of strings and
typically use gradients only as guidance. The Greedy Coordinate Gradient (GCG) attack \citep{zhouUniversal2023} is a representative
token-level method.

\section{Our Method}
\label{sec-method}
\begin{figure}[t]
\begin{minipage}[t]{0.49\linewidth}
\begin{algorithm}[H]
\caption{Grad-Predict FGSM (GPR-FGSM)}
\begin{algorithmic}[]
\Require Model $f_\theta$ with logits $z_\theta(\cdot)$, perturbation budget $\varepsilon$, inputs/targets $(x^\star_i,t^\star_i)$, predictor coefficients $A$, $b$ (feature-standardized).
\For{$(x^\star_i,t^\star_i)$}
    \State $r_i \gets h_\theta^{(\ell)}(x^\star_i)$ \Comment{forward only}
    \State $\hat g_i \gets A\, r_i + b$ \Comment{predict grad.\ direction}
    \State $x^{\mathrm{adv}}_i \gets x^\star_i + \varepsilon\,\sgn(\hat g_i)$
\EndFor
\end{algorithmic}
\label{alg-gp}
\end{algorithm}
\end{minipage}\hfill
\begin{minipage}[t]{0.49\linewidth}
\begin{algorithm}[H]
\caption{Vanilla FGSM }
\begin{algorithmic}[]
\Require Model $f_\theta$ with logits $z_\theta(\cdot)$, perturbation budget $\varepsilon$, inputs/targets $(x^\star_i,t^\star_i)$.
\State 
\For{$(x^\star_i,t^\star_i)$}
    \State 
    \State $g_i \gets \nabla_x S_\theta(x^\star_i,t^\star_i)$ \Comment{compute gradient}
    \State $x^{\mathrm{adv}}_i \gets x^\star_i + \varepsilon\,\sgn(g_i)$
    \vspace{2pt}
\EndFor
\end{algorithmic}
\label{alg-vanilla}
\end{algorithm}
\end{minipage}
\end{figure}

\paragraph{Main idea.}
Our method eliminates the backward pass at attack time by predicting the input-gradient direction from a forward-pass hidden state.
Concretely, we fit a linear predictor (with a bias term) on a small training set:
\begin{equation}
\nabla_x S_\theta(x,t) \approx\ A\,h_\theta(x)+b,
\label{eq-approx}
\end{equation}
where $h_\theta(x)$ denotes the chosen representation. The predictor is learned for a specified model, layer, objective, and prompt-target distribution. Targets are used when learning the predictor; we do not claim that arbitrary target changes can be inferred from early hidden states. We evaluate generalization to evaluation prompts (from the same distribution), we do not claim universal meta-predictions across models and objectives.

\paragraph{Proof structure.}
We analyze an idealized model in the NTK regime. Lemma~\ref{lem:deep-sincos-stationary} supplies stationarity for a trigonometric embedding, and Lemma~\ref{lem:stationary-implies-constant-sigma} gives its constant cross-covariances; neither establishes stationarity in the fully general setting of real-world transformers.  Gaussian conditioning then therefore gives an exact affine conditional mean at each input (Proposition~\ref{prop:affine-predictor}). All theoretical results are proved in Appendix \ref{app:proofs}. Our experiments test learned affine predictors using intermediate-layer representations and normalized targets beyond this idealization. Proofs appear in Appendix~\ref{app:proofs}.
\paragraph{Embedding choice.}
A sine-cosine embedding makes input translations equivalent to rotations of Gaussian coefficient pairs, giving a tractable stationary construction. 
This is an illustrative model, not a claim that real-world trained transformers are stationary in input-embedding space; Although RoPE also uses rotations \citep{su2024roformer}, it does it in a somewhat  different setup. The following lemma (see appendix \ref{app:proofs} for proof) shows that the sine-cosine embedding has a useful stationarity property.

\begin{lemma}[Stationarity]\label{lem:deep-sincos-stationary}
Fix $d,m\in\mathbb{N}$. Let $\omega_j\sim p$ on $\R^d$
and $(u_j,v_j)\sim\mathcal{N}(0,I_2)$, independently across
$j$ and between frequencies and coefficient pairs. Define
$e(x)=(e_1(x)^\top,\ldots,e_m(x)^\top)^\top$ by
\[
e_j(x)=
\begin{bmatrix}
\cos(\omega_j^\top x) & \sin(\omega_j^\top x)\\
-\sin(\omega_j^\top x) & \cos(\omega_j^\top x)
\end{bmatrix}
\begin{bmatrix}u_j\\v_j\end{bmatrix}.
\]
Let the remaining network parameters $\Theta$ be independent
of the embedding parameters, and set
\[
H(x)=R(e(x);\Theta),\qquad
Y(x)=w(\Theta)^\top H(x)+\beta(\Theta),
\]
where $R$ is any measurable map into $\R^{d_h}$.
Assume $H(x)$ and $Y(x)$ have finite second moments, and
$Y(\cdot)$ is almost surely differentiable with integrable
gradient $G(x):=\nabla_xY(x)$.

Then $(H,Y)$ is jointly stationary: its means are constant,
and its auto- and cross-covariances depend only on $x-x'$.
In particular, $k(x,x'):=\E[Y(x)Y(x')]$ is stationary.
The unconditional prior gradient mean
$\mu_G(x):=\E[G(x)]$ is also constant in $x$.\footnote{This prior mean is not the conditional gradient prediction
$\E[G(x)\mid H(x)]$, which need not be constant.}
\end{lemma}

\paragraph{Constant covariances.} In the following lemma (see appendix \ref{app:proofs} for a proof), we identify conditions under which the centered covariances in \eqref{eq:gauss-cond} are independent of $x$.

\begin{lemma}[Constant covariances]
\label{lem:stationary-implies-constant-sigma}
Let $H(x)$, $Y(x)$, and $G(x)=\nabla_xY(x)$ be as in
Lemma~\ref{lem:deep-sincos-stationary}, and assume additionally
that $\E\|G(x)\|^2<\infty$ for every $x$.
Then the same-input covariances
\[
\Sigma_{HH}(x):=\Cov(H(x),H(x)),
\qquad
\Sigma_{GH}(x):=\Cov(G(x),H(x))
\]
are independent of $x$.
\end{lemma}

\paragraph{Main result.} The following proposition derives the formula for the affine predictor, by characterizing an input-independent conditional mean in the idealized setting of joint Gaussianity and constant first and cross-covariance moments. Lemma~\ref{lem:stationary-implies-constant-sigma} supplies the covariance condition. The proposition is proved in appendix \ref{app:proofs}.

\begin{proposition}[Input-independent affine gradient predictor]
\label{prop:affine-predictor}
Suppose $H(x)$ and $G(x)$ have input-independent means
$\mu_H,\mu_G$ and same-input covariances
$\Sigma_{HH}=\Cov(H(x),H(x))$ and
$\Sigma_{GH}=\Cov(G(x),H(x))$, as supplied by
Lemmas~\ref{lem:deep-sincos-stationary}
and~\ref{lem:stationary-implies-constant-sigma}.
Assume additionally that $(H(x),G(x))$ is jointly Gaussian
for each $x$ and that $\Sigma_{HH}$ is invertible.

Then
\[
\E[G(x)\mid H(x)]=AH(x)+b,
\qquad
A=\Sigma_{GH}\Sigma_{HH}^{-1},
\quad b=\mu_G-A\mu_H,
\]
with $A,b$ independent of $x$.
For each fixed $x$, this is the exact minimum-mean-squared-error
predictor among all measurable functions of $H(x)$.
It recovers $G(x)$ almost surely precisely when
\[
\Cov(G(x),G(x))
-\Sigma_{GH}\Sigma_{HH}^{-1}\Sigma_{GH}^{\top}=0.
\]
\end{proposition}

The analysis above applies in the idealized NTK/Gaussian setup. Without Gaussianity, however, this predictor need not equal $\E[G\mid H]$: a nonlinear predictor may do better. However, we can show that the derived predictor is at least best in class of all affine predictors. We do that in the following observation (see appendix \ref{app:proofs}) for a proof.

\begin{observation}[Best affine predictor]
\label{obs-best-affine}
Let $H$ and $G$ be any square-integrable random vectors
with $\Cov(H,H)$ invertible. Then
\[
A=\Cov(G,H)\Cov(H,H)^{-1},
\qquad
b=\E[G]-A\E[H]
\]
minimize $\E\|G-AH-b\|_2^2$ over all affine predictors.
Neither Gaussianity nor stationarity is required.
\end{observation}

\paragraph{Predictor Training.}
The observation above applies to a fixed trained model with prompts sampled form a fixed distribution, with normalized gradient targets. In other words, the predictor has to be retrained for each use-case.

\section{Related Work}

\paragraph{Gradient-based attacks.}
Adversarial training aims to improve robustness by incorporating worst-case (or approximately worst-case) perturbations into the learning objective. Early work demonstrated that neural networks admit adversarial perturbations found via optimization \citep{szegedyIntriguing2014}. \citet{goodfellowAdversarial2015} introduced the Fast Gradient Sign Method (FGSM), enabling efficient single-step attacks using input gradients, while \citet{madry2018} formulated adversarial robustness as a min–max problem and advocated multi-step projected gradient descent (PGD). A large body of follow-up work focuses on reducing the cost of these gradient computations. For example, adversarial training in the style of  \citet{shafahi2019adversarialtrainingfree} recycles the gradient computed for the weight update to also construct the adversarial perturbation, YOPO restricts backpropagation to the first layer \citep{zhang2019propagateonceacceleratingadversarial}, randomly initialized FGSM can match PGD-based adversarial training when combined with appropriate training hyperparameters \citep{Wong2020Fast} and single-step methods with stabilization techniques (e.g., GradAlign) approximate multi-step robustness \citep{andriushchenko2020understanding}. Despite these improvements, such methods fundamentally rely on backpropagation to obtain input gradients. Our approach builds on this line of work, with the goal of making the methods faster by using synthetic gradients.

\paragraph{Gradient-free (black-box) attacks.}
An alternative paradigm eliminates backpropagation by estimating gradients through input perturbations. Zeroth-order optimization methods such as ZOO, NES and SPSA approximate $\nabla_x S(x)$, where $S(x)$ is the attack objective (e.g., a target logit or loss), using randomized queries \citep{chenzoo2017, ilyas2018blackbox, pmlr-v80-uesato18a}. These approaches are applicable in black-box settings but typically require hundreds to thousands of forward evaluations per example, trading gradient access for substantial query complexity. Like our approach, they avoid backpropagation; however, they estimate gradients via repeated input queries, whereas we predict them directly from forward-pass representations.

\paragraph{Gradient prediction.}
A separate line of work replaces gradient computation altogether by learning predictors of gradients. Synthetic gradient methods \cite{jaderberg2017decoupled, czarnecki2017understanding} decouple neural network modules by predicting gradients during training. Our approach adapts this paradigm to adversarial attacks by predicting input gradients directly from forward-pass representations. In our setting, each gradient estimate uses forward-pass features and a linear map without backpropagation or query-based gradient estimation. Multi-step attacks repeat this operation, and GCG additionally evaluates candidate losses.

\paragraph{Steering.} Steering methods seek to controllably alter a model's behavior at inference time, often by intervening on internal representations rather than updating weights. Inference-Time Intervention (ITI) \citep{liintervention2023} learns a small set of activation-space directions (applied to selected attention heads) and shifts activations along these directions during the forward pass to elicit more truthful responses, revealing a tunable tradeoff between truthfulness and other desired behaviors. Representation Engineering \citep{zouRepresentation2023} advocates a top-down view in which representations encode high-level concepts; it develops practical procedures to discover, monitor, and manipulate directions relevant to transparency and safety (e.g., honesty or harmfulness). Activation Addition \citep{turner2024activation} provides a lightweight steering primitive that computes a steering vector from differences in intermediate activations induced by contrastive prompt pairs and then adds this vector during inference, enabling rapid, optimization-free control over attributes such as sentiment or toxicity while aiming to preserve performance on off-target tasks. 
Our approach is related in that it also leverages internal representations at inference time. However, whereas steering methods modify hidden states to control behavior, we instead use hidden states as predictors of input gradients. In this sense, steering performs representation-space intervention, while our method performs representation-based gradient inference.

\paragraph{Neural Tangent Kernel (NTK).} The infinite-width description distinguishes the initialization covariance (NNGP) kernel from the NTK governing gradient-based training \citep{jacot2018neural}. \citet{lee2018deep} and \citet{matthews2018gaussian} derive the GP limit of randomly initialized deep fully-connected networks, and \citet{yang2019wide} extends such results to broader architecture classes. For our theoretical analysis, we use the NTK approach, not Bayesian posterior inference with the NNGP kernel. Specifically,  we leverage the fact that, under training on fixed data, in the NTK setting,  the limiting output process remains Gaussian \citep{jacot2018neural}. \footnote{See appendix~\ref{app:gp-derivatives} for details in our notation.} 

\section{Experiments}
\label{sec-experiments}

\paragraph{Structure of Experiments} We structure our experimental results into two sections, each considering a different type of attack on the LLM. Section \ref{sec-attack-embedding} deals with attacks in a setting where we can perturb the input embedding, while section \ref{sec-attack-token} focuses on attacks via exchanging input tokens.

\paragraph{LLMs and Datasets} We wanted to evaluate for several different levels of model hardening quality so we used three LLMs: Llama 2 (released in 2023), Qwen 2.5 (released in 2024) and Qwen 3 (released in 2025). We tested all 3 models against JailbreakBench \citep{chao2024jailbreakbench} and HarmBench \citep{mazeika2024harmbench} -- both of which are leading open-source frameworks designed to evaluate and enhance the robustness of LLMs against jailbreak attacks.

\paragraph{Practical modifications.}
In practice, we use several simple technical modifications that substantially improve speed and stability. First, rather than extracting $h_\theta(x)$ from the final layer, we regress on an intermediate hidden state $h_\theta^{(\ell)}(x)$ and (at test time) obtain it via an early-exit forward pass, reducing compute while still preserving a strong signal for predicting input gradients. Second, we normalize regression targets per token to emphasize gradient direction rather than magnitude. Third, we standardize hidden-state features coordinatewise, which is standard practice for linear regression with heterogeneous feature scales. Finally, we fit the predictor with ridge regression, which both stabilizes the solution in the presence of collinearity and provides a simple form of regularization that improves generalization from small training sets.

\subsection{Embedding-level attacks.}

\begin{table}[t]
    \centering \footnotesize 
    \caption{Attack success rate (ASR), perturbation throughput (Pert.), and success throughput (Succ.) for embedding-level attacks on held-out prompts from HarmBench. 95\% bootstrap CIs from per-prompt resampling ($B=10{,}000$).}
    \label{tab:main}
    \setlength{\tabcolsep}{2pt}
    \resizebox{\linewidth}{!}{
\begin{tabular}{@{}l rrr rrr rrr@{}}
\toprule
 & \multicolumn{3}{c}{Qwen3-4B} & \multicolumn{3}{c}{Qwen2.5-3B} & \multicolumn{3}{c}{Llama-2-7B} \\
\cmidrule(lr){2-4} \cmidrule(lr){5-7} \cmidrule(lr){8-10}
Method & ASR\,(\%) & Pert.\,(ex/s) & Succ.\,(att/s) & ASR\,(\%) & Pert.\,(ex/s) & Succ.\,(att/s) & ASR\,(\%) & Pert.\,(ex/s) & Succ.\,(att/s) \\
\midrule 
\multicolumn{10}{c}{Exact gradient} \\
\cmidrule{1-10}
FGSM & $12.6_{-3.6}^{+3.9}$ & $5.1_{-0.0}^{+0.0}$ & $0.65_{-0.18}^{+0.20}$ & $24.5_{-4.8}^{+4.8}$ & $6.3_{-0.1}^{+0.0}$ & $1.55_{-0.31}^{+0.31}$ & $18.6_{-4.3}^{+4.6}$ & $7.4_{-0.0}^{+0.0}$ & $1.37_{-0.32}^{+0.35}$ \\
FGM & $9.7_{-3.2}^{+3.6}$ & $5.1_{-0.0}^{+0.0}$ & $0.50_{-0.17}^{+0.18}$ & $6.1_{-2.7}^{+3.1}$ & $6.1_{-0.0}^{+0.0}$ & $0.38_{-0.17}^{+0.19}$ & $10.0_{-3.2}^{+3.6}$ & $7.4_{-0.0}^{+0.0}$ & $0.74_{-0.24}^{+0.26}$ \\
RS-FGSM & $11.0_{-3.2}^{+3.6}$ & $5.1_{-0.0}^{+0.0}$ & $0.56_{-0.17}^{+0.18}$ & $22.8_{-4.8}^{+4.8}$ & $6.3_{-0.0}^{+0.0}$ & $1.43_{-0.30}^{+0.30}$ & $18.6_{-4.6}^{+4.6}$ & $7.3_{-0.0}^{+0.0}$ & $1.35_{-0.34}^{+0.34}$ \\
PGD-2 & $16.5_{-4.2}^{+4.2}$ & $2.6_{-0.0}^{+0.0}$ & $0.42_{-0.11}^{+0.11}$ & $31.0_{-5.1}^{+5.4}$ & $3.1_{-0.0}^{+0.0}$ & $0.97_{-0.16}^{+0.17}$ & $39.3_{-5.7}^{+5.7}$ & $3.7_{-0.0}^{+0.0}$ & $1.44_{-0.21}^{+0.21}$ \\
PGD-5 & $32.4_{-5.2}^{+5.5}$ & $1.0_{-0.0}^{+0.0}$ & $0.33_{-0.05}^{+0.06}$ & $59.9_{-5.8}^{+5.5}$ & $1.3_{-0.0}^{+0.0}$ & $0.76_{-0.07}^{+0.07}$ & $41.4_{-5.7}^{+5.7}$ & $1.5_{-0.0}^{+0.0}$ & $0.61_{-0.08}^{+0.08}$ \\
PGD-10 & $35.6_{-5.2}^{+5.5}$ & $0.5_{-0.0}^{+0.0}$ & $0.18_{-0.03}^{+0.03}$ & $72.1_{-5.1}^{+5.1}$ & $0.6_{-0.0}^{+0.0}$ & $0.46_{-0.03}^{+0.03}$ & $49.3_{-5.7}^{+5.7}$ & $0.7_{-0.0}^{+0.0}$ & $0.36_{-0.04}^{+0.04}$ \\
\addlinespace
\multicolumn{10}{c}{GPR (predicted)} \\
\cmidrule{1-10}
FGSM & $6.8_{-2.6}^{+2.9}$ & $60.5_{-0.2}^{+0.2}$ & $\mathbf{4.11}_{-1.57}^{+1.77}$ & $11.6_{-3.4}^{+3.7}$ & $73.4_{-0.3}^{+0.2}$ & $\mathbf{8.49}_{-2.51}^{+2.74}$ & $15.4_{-3.9}^{+4.3}$ & $47.4_{-0.2}^{+0.2}$ & $\mathbf{7.28}_{-1.87}^{+2.04}$ \\
FGM & $4.2_{-1.9}^{+2.3}$ & $59.6_{-0.3}^{+0.2}$ & $\mathbf{2.51}_{-1.16}^{+1.36}$ & $3.4_{-2.0}^{+2.4}$ & $71.4_{-0.2}^{+0.2}$ & $\mathbf{2.43}_{-1.45}^{+1.69}$ & $2.1_{-1.4}^{+1.8}$ & $47.0_{-0.1}^{+0.1}$ & $\mathbf{1.01}_{-0.67}^{+0.84}$ \\
RS-FGSM & $8.1_{-2.9}^{+3.2}$ & $60.3_{-0.2}^{+0.2}$ & $\mathbf{4.88}_{-1.75}^{+1.92}$ & $13.6_{-3.7}^{+4.1}$ & $72.3_{-0.3}^{+0.3}$ & $\mathbf{9.84}_{-2.71}^{+2.95}$ & $16.4_{-4.3}^{+4.3}$ & $47.1_{-0.1}^{+0.1}$ & $\mathbf{7.74}_{-2.01}^{+2.04}$ \\
PGD-2 & $6.5_{-2.6}^{+2.9}$ & $30.8_{-0.1}^{+0.1}$ & $\mathbf{2.00}_{-0.80}^{+0.89}$ & $7.5_{-2.7}^{+3.1}$ & $37.2_{-0.1}^{+0.1}$ & $\mathbf{2.78}_{-1.01}^{+1.15}$ & $13.6_{-3.9}^{+3.9}$ & $23.9_{-0.1}^{+0.1}$ & $\mathbf{3.25}_{-0.94}^{+0.95}$ \\
PGD-5 & $6.5_{-2.6}^{+2.9}$ & $12.5_{-0.0}^{+0.0}$ & $\mathbf{0.81}_{-0.32}^{+0.36}$ & $17.3_{-4.4}^{+4.4}$ & $15.2_{-0.0}^{+0.0}$ & $\mathbf{2.63}_{-0.67}^{+0.68}$ & $15.7_{-3.9}^{+4.3}$ & $9.7_{-0.0}^{+0.0}$ & $\mathbf{1.52}_{-0.38}^{+0.42}$ \\
PGD-10 & $3.9_{-1.9}^{+2.3}$ & $6.3_{-0.0}^{+0.0}$ & $\mathbf{0.25}_{-0.12}^{+0.14}$ & $18.7_{-4.4}^{+4.4}$ & $7.7_{-0.0}^{+0.0}$ & $\mathbf{1.43}_{-0.34}^{+0.34}$ & $15.4_{-3.9}^{+4.3}$ & $4.9_{-0.0}^{+0.0}$ & $\mathbf{0.75}_{-0.19}^{+0.21}$ \\
\bottomrule
\end{tabular}
}
\end{table}

\label{sec-attack-embedding}
\paragraph{Setup}
In this set of experiments, we perturb the continuous token embeddings which the model receives as input. Given a prompt\footnote{We only consider prompts for which the zero-shot method doesn't work i.e.,\ where the model refuses the harmful behavior without an intervention.}, we search for a small perturbation of its embedding sequence that causes the model to produce a harmful response instead of refusal. The perturbation is bounded by an $\epsilon$-ball of the original embedding sequence, according to $L_2$ or $L_{\infty}$ norm, depending on method. Our main optimization criterion is \emph{successful attack throughput}, i.e., the number of prompts for which we successfully find an adversarial embedding perturbation per second. Note that with this criterion we may prefer a method that has lower \emph{attack success rate (ASR)} (i.e.,  the proportion of prompts for which adversarial perturbation is found) over the method with better ASR, but lower processing speed. When measuring successful attack throughput, we only include time required to generate embedding perturbation. 

\paragraph{Predictor Training Details} The GPR predictor is a single linear map (with a bias feature) from a layer-$L$ hidden state to the per-token, unit-normalized gradient $\hat g_t = g_t/\lVert g_t\rVert_2$ of the target logit with respect to the input embedding at token $t$; each input-token position contributes one training sample. For every training prompt we (i) run a full forward/backward pass and record layer-$L$ input hidden states paired with $\hat g_t$, and (ii) take $K$ FGSM augmentation steps in input-embedding space, recomputing the forward/backward at each step and recording the resulting $(h, \hat g)$ pairs. Step-$k$ samples are downweighted by $\gamma^{k}$. Hidden-state features are standardized per dimension using the empirical mean and std over all stored samples and then concatenated with a constant bias column. The weights are obtained in closed form by a single weighted ridge regression with fixed $\lambda{=}1$, solved via an eigendecomposition, fitting all $d$ output dimensions jointly. The choice of hyperparameters is described in appendix \ref{sec-hypers-emb}, together with a benchmark reporting how sensitive we are to hyperparameter choice.

\paragraph{Predictor Training Cost}
Gradient prediction removes backward passes at attack time, not when training the predictor. For example, for the Qwen2.5-3B FGM experiment, this takes 159.2-259.7 seconds (see \ref{sec:embedding-transfer} for details).
Comparing the wall-clock time of our method ($T_{\mathrm{cal}} + N\bar t_{\mathrm{GPR}}$) with the wall-clock time of the baseline($N\bar t_{\mathrm{exact}}$) gives an equal-attempt runtime break-even point of between 1,280 and 2,113 attacked examples. 

\paragraph{Main Results}
Table~\ref{tab:main} reports results on model-specific non-trivial subsets of the 320-prompt HarmBench test split. GPR has higher successful-attack-throughput point estimates than the corresponding exact-gradient method in each reported HarmBench comparison, but the gains vary substantially and ASR is often lower. For example, FGSM on Qwen3-4B increases successful-attack throughput from $0.65$ to $4.11$ attacks/s, while ASR decreases from $12.6\%$ to $6.8\%$. Appendix~\ref{sec-exp-throughput} provides a direct throughput plot. Appendix B.3 reports cross-benchmark transfer, cosine/sign-agreement diagnostics, and predictor training cost for a controlled FGM configuration.

\paragraph{Throughput-ASR Pareto frontier.}
Figure~\ref{fig:pareto} plots ASR against successful-attack throughput across all attack variants and models. It can be seen that GPR enables a new part of the Pareto frontier, with the largest gains for the cheap one-shot attacks (FGSM, RS-FGSM, FGM). The trade-off between ASR and throughput is further visible in figure~\ref{fig:pgd_scaling}, which plots both as a function of how many PGD steps the method takes. As can be expected, GPR works less well (but still has a higher throughput) for a larger number of steps. We theorize this is because of compounding of errors across several gradient steps.

\begin{figure}[t]
    \centering
    \includegraphics[width=\linewidth]{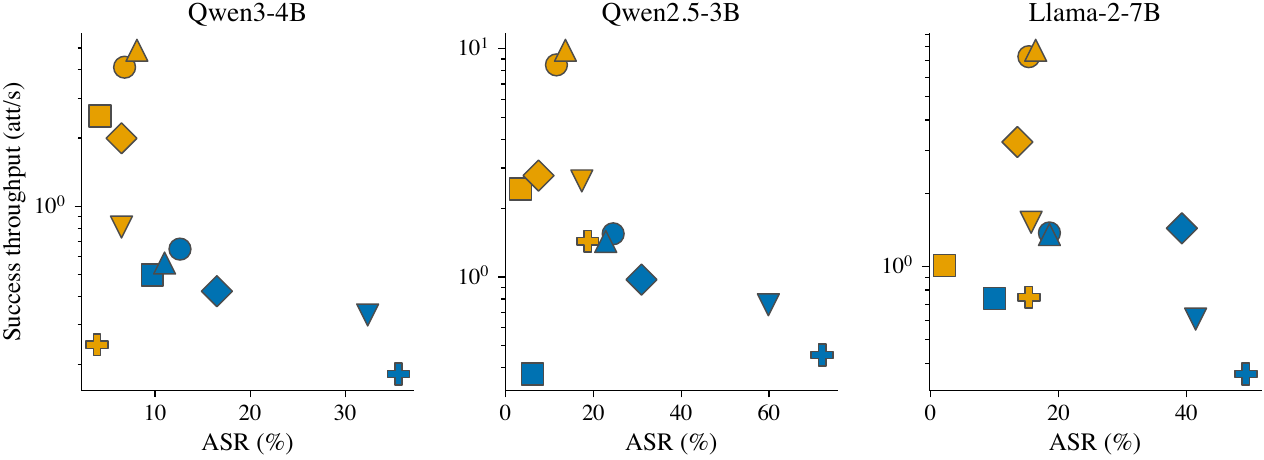}
    \caption{Pareto frontier of ASR vs.\ successful-attack throughput on HarmBench. Each point is a (method, model) configuration. Gradient source (color): \dotexact~exact, \dotgpr~GPR.
Attack (shape): \mkFGSM~FGSM, \mkFGM~FGM, \mkRSFGSM~RS-FGSM,
\mkPGDtwo~PGD-2, \mkPGDfive~PGD-5, \mkPGDten~PGD-10.}
    \label{fig:pareto}
\end{figure}

\begin{figure}[t]
    \centering
    \begin{subfigure}{0.48\linewidth}
        \includegraphics[width=\linewidth]{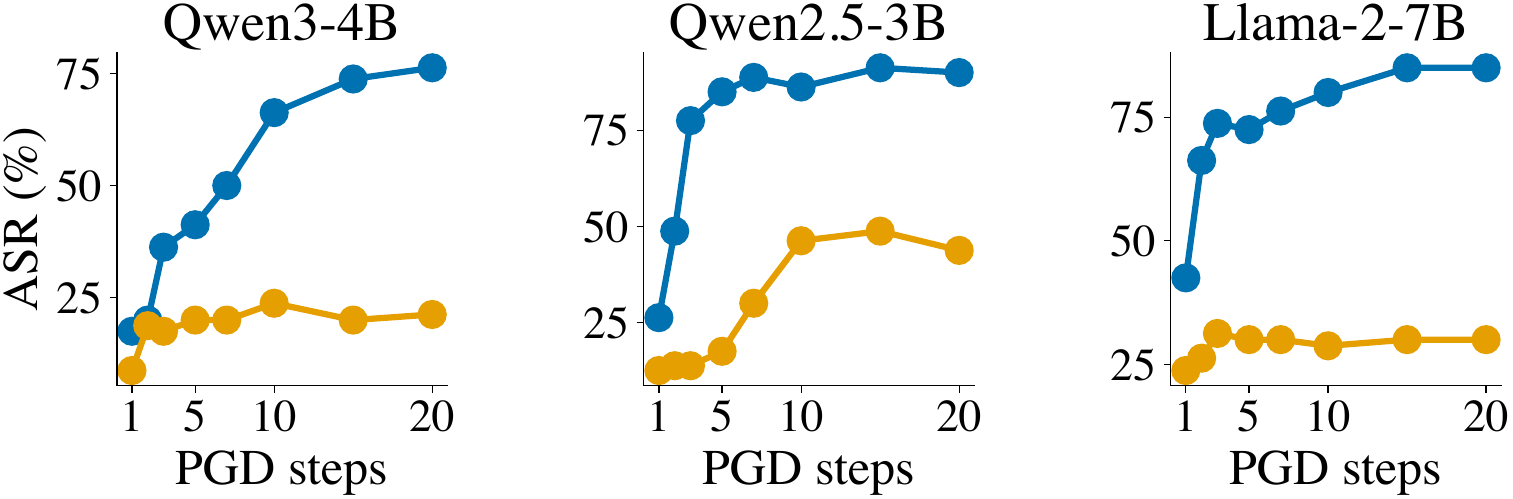}
        \caption{ASR vs.\ no.\ of PGD steps.}
    \end{subfigure}%
    \hfill
    \begin{subfigure}{0.48\linewidth}
        \includegraphics[width=\linewidth]{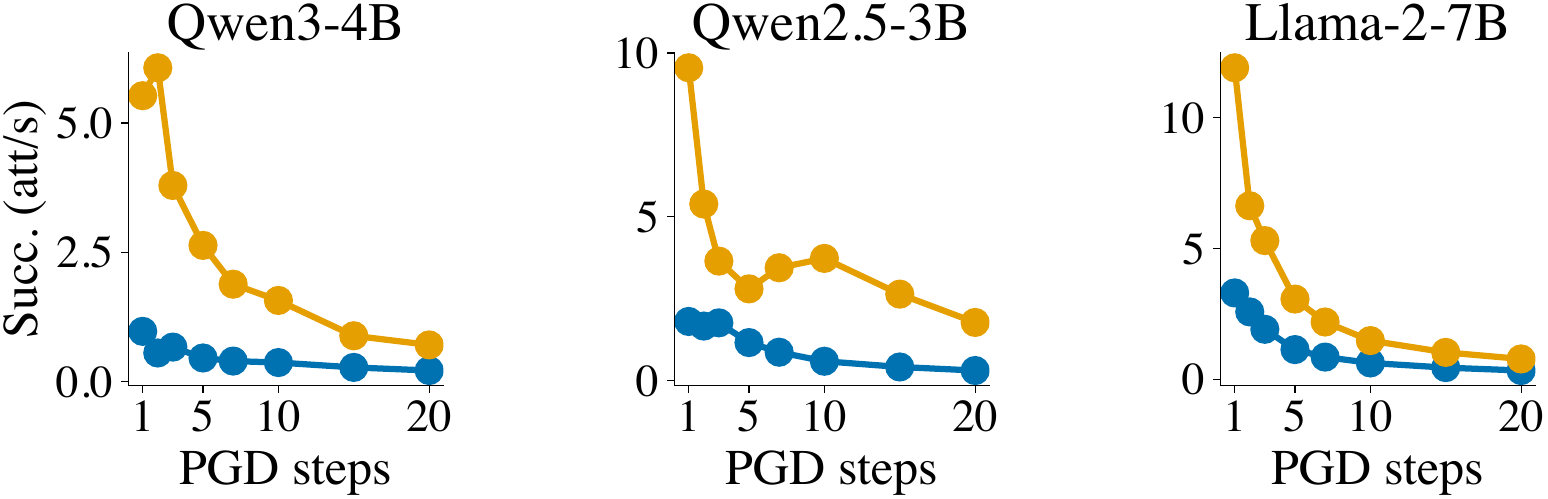}
        \caption{Successful-attack throughput vs.\ no.\ of PGD steps.}
    \end{subfigure}
    \caption{PGD scaling on HarmBench. Curves: \dotexact~PGD with exact gradient, \dotgpr~PGD-GPR}.
    \label{fig:pgd_scaling}
\end{figure}


\subsection{Token-level attacks.}
\label{sec-attack-token}

\paragraph{Setup.}
We evaluate whether gradient prediction can accelerate discrete prompt attacks, where the adversary optimizes an adversarial token suffix rather than a continuous perturbation. We use Greedy Coordinate Gradient (GCG) \citep{zhouUniversal2023}, which iteratively uses token-embedding gradients to propose single-token replacements and then selects among those candidates using exact forward losses on a target continuation. Our variant, GCG-GPR, keeps the same GCG objective, candidate evaluation, and acceptance rule, but replaces each backward-computed suffix gradient with a predicted gradient from an intermediate hidden representation. Thus, the predictor only affects which candidate tokens are considered; the accepted update is still chosen using the true autoregressive loss. Appendix~\ref{app:token-attack-details} gives the full objective and candidate-generation equations.

\paragraph{Experimental setting}
We compare exact-gradient GCG and GCG-GPR on HarmBench and JailbreakBench using the same suffix-optimization protocol. Both methods optimize the same target-continuation objective and use exact forward losses to select token substitutions; they differ only in whether candidate tokens are proposed using an exact or predicted gradient. We additionally evaluate a blind random-suffix baseline as a gradient-free sanity check. For each prompt, this baseline independently replaces every position in the GCG suffix region with a token sampled uniformly from the tokenizer's non-special vocabulary and uses the resulting suffix directly, without gradient computation, candidate-loss evaluation, or iterative search. \footnote{It tests whether arbitrary suffix perturbations suffice; it is not a GCG variant with random candidate proposals, since it does not retain exact-loss selection.}
We fit the predictor on augmented suffix states generated from GCG trajectories, so that its training inputs resemble the iterates encountered during the attack. We report ASR for all methods and perturbation and successful-attack throughput for exact-gradient GCG and GCG-GPR. Throughput is left undefined for the blind random baseline because it performs no optimization. Appendix~\ref{app:token-attack-hparams} provides the complete experimental settings.

\paragraph{Results}

\begin{table}[t]
\caption{Attack success rate (ASR), perturbation throughput (Pert.), and success throughput (Succ.) for prompt-level GCG attacks on held-out prompts from HarmBench (HB) and JailbreakBench (JBB). The blind random baseline independently samples one suffix per prompt and performs no loss-based search; its Pert. and Succ. entries are therefore left undefined. 95\% bootstrap CIs from per-prompt resampling ($B=10{,}000$).}
\label{tab:gcg_prompt_results_table}
\centering
\setlength{\tabcolsep}{2pt}
\footnotesize
\resizebox{\linewidth}{!}{\begin{tabular}{@{}ll rrr rrr rrr@{}}
\toprule
 &  & \multicolumn{3}{c}{Qwen3-4B} & \multicolumn{3}{c}{Qwen2.5-3B} & \multicolumn{3}{c}{Llama-2-7B} \\
\cmidrule(lr){3-5} \cmidrule(lr){6-8} \cmidrule(lr){9-11}
 & Method & ASR\,(\%) & Pert.\,(ex/s) & Succ.\,(att/s) & ASR\,(\%) & Pert.\,(ex/s) & Succ.\,(att/s) & ASR\,(\%) & Pert.\,(ex/s) & Succ.\,(att/s) \\
\midrule
\multicolumn{11}{c}{Blind random suffix (no search)} \\
\cmidrule{1-11}
HB & Random & $0.0_{-0.0}^{+0.0}$ & \textemdash & \textemdash & $5.9_{-2.5}^{+2.8}$ & \textemdash & \textemdash & $8.1_{-2.8}^{+3.1}$ & \textemdash & \textemdash \\
JBB & Random & $0.0_{-0.0}^{+0.0}$ & \textemdash & \textemdash & $1.0_{-1.0}^{+2.0}$ & \textemdash & \textemdash & $6.0_{-4.0}^{+5.0}$ & \textemdash & \textemdash \\
\addlinespace
\multicolumn{11}{c}{Exact gradient} \\
\cmidrule{1-11}
HB & GCG & $5.0_{-2.2}^{+2.5}$ & $0.1_{-0.0}^{+0.0}$ & $0.01_{-0.00}^{+0.00}$ & $49.4_{-5.3}^{+5.3}$ & $0.2_{-0.0}^{+0.0}$ & $0.08_{-0.01}^{+0.01}$ & $17.2_{-4.1}^{+4.4}$ & $0.1_{-0.0}^{+0.0}$ & $0.02_{-0.01}^{+0.01}$ \\
JBB & GCG & $0.0_{-0.0}^{+0.0}$ & $0.1_{-0.0}^{+0.0}$ & $0.00_{-0.00}^{+0.00}$ & $60.0_{-10.0}^{+10.0}$ & $0.2_{-0.0}^{+0.0}$ & $0.10_{-0.02}^{+0.02}$ & $21.0_{-8.0}^{+8.0}$ & $0.1_{-0.0}^{+0.0}$ & $0.02_{-0.01}^{+0.01}$ \\
\addlinespace
\multicolumn{11}{c}{GPR (predicted)} \\
\cmidrule{1-11}
HB & GCG & $4.4_{-2.2}^{+2.5}$ & $0.3_{-0.0}^{+0.0}$ & $0.01_{-0.01}^{+0.01}$ & $50.9_{-5.3}^{+5.6}$ & $0.3_{-0.0}^{+0.0}$ & $\mathbf{0.18}_{-0.02}^{+0.02}$ & $19.4_{-4.4}^{+4.4}$ & $0.2_{-0.0}^{+0.0}$ & $\mathbf{0.04}_{-0.01}^{+0.01}$ \\
JBB & GCG & $0.0_{-0.0}^{+0.0}$ & $0.3_{-0.0}^{+0.0}$ & $0.00_{-0.00}^{+0.00}$ & $69.0_{-9.0}^{+9.0}$ & $0.3_{-0.0}^{+0.0}$ & $\mathbf{0.24}_{-0.03}^{+0.03}$ & $21.0_{-8.0}^{+8.0}$ & $0.2_{-0.0}^{+0.0}$ & $\mathbf{0.04}_{-0.01}^{+0.01}$ \\
\bottomrule
\end{tabular}
}
\end{table}

Table~\ref{tab:gcg_prompt_results_table} reports ASR for the blind random baseline and ASR and throughput for exact-gradient GCG and GCG-GPR. The blind random suffix is substantially less effective in every setting in which GCG-GPR obtains nonzero ASR. This confirms that arbitrary, unoptimized suffix perturbations do not account for the effectiveness of the predicted-gradient attack.

Comparing the two optimized attacks, GCG-GPR has comparable ASR point estimates to exact-gradient GCG, with overlapping confidence intervals in the reported settings and higher point estimates on Qwen2.5-3B on both benchmarks. Replacing the backward-computed gradient increases perturbation throughput by roughly \(1.5\)--\(3\times\), yielding \(2\)--\(2.4\times\) higher successful-attack throughput in the nonzero-success regimes. On Qwen3-4B with JailbreakBench, all three methods obtain \(0\%\) ASR. Overall, predicted gradients produce substantially more effective suffixes than blind random sampling while retaining the efficiency advantage over exact-gradient GCG.

\subsection{Sensitivity to gradient-prediction quality}
\label{sec:sensitivity}

The embedding-level main results in Section~\ref{sec-attack-embedding} use
HarmBench. We additionally evaluate the same attacks on JailbreakBench
(Appendix~\ref{sec-emb-jbb}, Table~\ref{tab:main-jbb}), obtaining less favorable
results. On Qwen2.5-3B, FGM-GPR has observed ASR
$\JBBQwenTwoFGMGPRASR\%$, compared with $\JBBQwenTwoFGMExactASR\%$ for FGM,
while FGSM-GPR has ASR $\JBBQwenTwoFGSMGPRASR\%$, compared with
$\JBBQwenTwoFGSMExactASR\%$ for FGSM. The effect of gradient prediction on ASR
varies across the evaluated attack configurations. Low ASR in some
exact-gradient baselines also highlights limitations of the underlying
attacks.

To diagnose this, we run a mixed-gradient ablation: at each FGSM step we use
$g_{\mathrm{mix}} = (1-\alpha)\,\hat g + \alpha\, g_{\mathrm{exact}}$ for $\alpha \in \{0, 0.1, 0.25, 0.5, 0.75, 0.9, 1\}$, holding all hyperparameters at their GPR-tuned values so that gradient quality is the only varying factor. Figure~\ref{fig:alpha_sweep} plots ASR vs.~$\alpha$ on HarmBench and JailbreakBench for Qwen2.5-3B and Llama-2-7B.

\begin{figure}[t]
    \centering
    \includegraphics[width=0.8\linewidth]{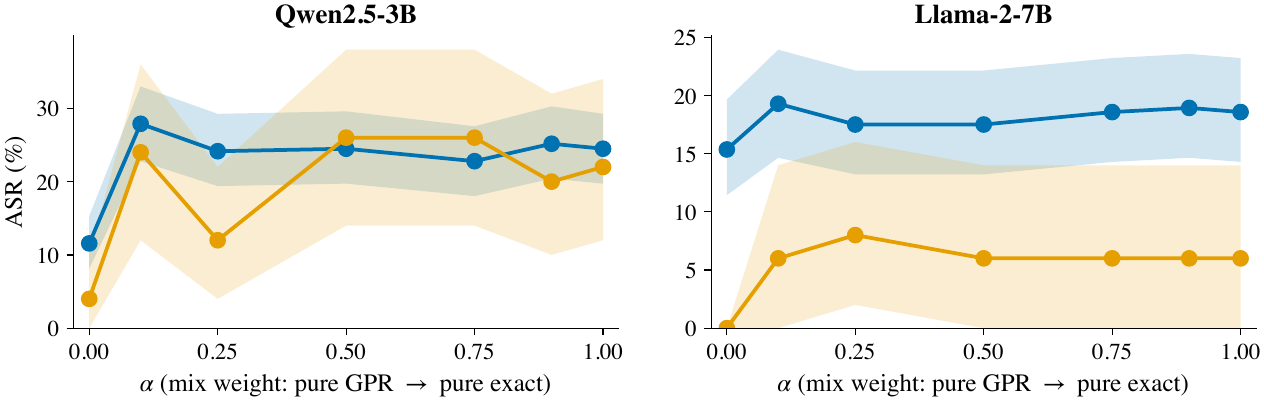}
    \input{aleksandrv/mixed_gradient_caption.tex}
    \label{fig:alpha_sweep}
\end{figure}

The effect of mixing depends on the model and benchmark. On Llama-2/JBB,
$\alpha=0.1$ gives observed ASR $\AlphaJBBLlamaTwoSmallASR\%$, compared with
$\AlphaJBBLlamaTwoZeroASR\%$ at $\alpha=0$ and
$\AlphaJBBLlamaTwoOneASR\%$ at $\alpha=1$; the Qwen2.5/JBB curve is less
flat. Intermediate mixtures can be useful in the evaluated settings.
Further evidence would be needed to establish a universal sign-agreement
threshold. The influence of each
term also depends on its gradient norm. This diagnostic experiment requires
backpropagation to compute exact gradients.
Figures~\ref{fig:cosine_asr_layer} and~\ref{fig:cosine_llama2}, together with
the cosine/sign-agreement results in Appendix B.3, provide complementary
diagnostics.

\section{Limitations}
Our experiments cover three open LLMs and two jailbreak benchmarks. The theory is a conditional-mean characterization under idealized assumptions, not an exact-gradient guarantee for real-world transformers. Predictors are trained for a particular model, layer, objective, and prompt-target distribution; model updates or distribution changes may require retraining the predictor. The cost of such training must be amortized. Also, embedding perturbations are not directly deployable token attacks. 

\section{Conclusions}
We presented a simple approach for accelerating adversarial example generation by replacing the per-example backward pass with a lightweight, learned linear predictor that maps forward-pass hidden states to input-gradient directions. Our method produces FGSM-style perturbations using only forward computation, substantially improving adversarial throughput under realistic wall-clock constraints. Theoretically, we derived an exact affine conditional-mean characterization for differentiable NTK-limit outputs. Empirically, the resulting gradient-prediction attack recovers a large fraction of the effectiveness of vanilla FGSM while using a small fraction of its compute, suggesting that representation-gradient relationships in modern networks are regular enough to exploit for fast adversarial generation.



\subsection*{Ethics statement}
This work lowers the computational cost of adversarial generation. It may support authorized robustness evaluation and red-teaming, but can also lower the cost of searching for harmful outputs or bypassing safeguards. All experiments use public models and datasets. The efficiency improvement therefore has dual-use implications, even though the underlying attack objectives already exist. The release is intended to support reproducible safety research.

\subsection*{AI use statement}
\label{sec:llm-usage}
Large language models were used during the preparation of this work in the following ways: (1) as coding assistants to help generate implementation code for the experiments, (2) to polish the writing of the manuscript, (3) in the development of the proofs, (4) to suggest initial wording of some paragraphs (which were then edited manually), (5) for finding related work (and getting a detailed understanding of the NTK/NNGP literature), (6) for research ideation and helping decide what experiments to run next. In all cases, the human authors verified the output: code was tested and reviewed, written passages were checked and edited, and proof steps were independently checked for correctness before inclusion. The authors take full responsibility for the content of the paper.

\bibliographystyle{iclr2027/iclr2027_conference}
\bibliography{refs}

\clearpage
\appendix

\section{Experimental Details - token-level attacks}\label{app:token-attack-details}
This appendix gives the full token-level attack setup used in Section~\ref{sec-attack-token}.
\subsection{Setup and Method Details}
\paragraph{Setup.}
We evaluate token-level attacks using Greedy Coordinate Gradient (GCG) \citep{zhouUniversal2023}, a discrete suffix-optimization attack for autoregressive language models. Let $q_{1:m}$ denote the fixed user prompt and let $a_{1:\ell}\in [V]^\ell$ denote an adversarial suffix of $\ell$ tokens, where $V$ is the vocabulary size. The full model input is the concatenation
\[
x(q,a) \;=\; [q_{1:m}\,,a_{1:\ell}].
\]
For each prompt, GCG specifies a short target continuation $t_{1:H}$, typically an affirmative prefix that begins the response in a compliant form, like ``Sure, here is''. The attack objective is the teacher-forced negative log-likelihood of this target continuation:
\begin{equation}
\mathcal L_\theta(q,a,t)
\;=\;
-\sum_{h=1}^{H}
\log P_\theta\!\left(t_h \mid x(q,a), t_{<h}\right).
\label{eq:gcg-loss}
\end{equation}
Thus, the token-level adversarial problem is
\begin{equation}
\min_{a_{1:\ell}\in [V]^\ell}\;
\mathcal L_\theta(q,a,t).
\label{eq:gcg-discrete-objective}
\end{equation}
In the universal setting, the same suffix is optimized across prompts, and optionally across models with a shared tokenizer, by minimizing an aggregate objective
\begin{equation}
\mathcal L(a)
\;=\;
\sum_{\theta'\in\mathcal M}
\sum_{j=1}^{J}
\lambda_{\theta',j}\,
\mathcal L_{\theta'}\!\left(q^{(j)},a,t^{(j)}\right),
\label{eq:gcg-aggregate-loss}
\end{equation}
where $\mathcal M$ is the set of attacked models and $\lambda_{\theta',j}$ are optional weights.

\paragraph{Greedy Coordinate Gradient.}
GCG optimizes \eqref{eq:gcg-discrete-objective} by repeatedly proposing single-token substitutions. Let $E\in\mathbb R^{V\times d_e}$ be the token embedding matrix, and let $u_r = E_{a_r}\in\mathbb R^{d_e}$ be the embedding of the $r$-th suffix token. If $a_{r\leftarrow v}$ denotes the suffix obtained by replacing token $a_r$ with candidate token $v$, then a first-order approximation gives
\begin{equation}
\mathcal L(a_{r\leftarrow v})
-
\mathcal L(a)
\;\approx\;
\left(E_v - E_{a_r}\right)^\top
\nabla_{u_r}\mathcal L(a).
\label{eq:gcg-linearized-change}
\end{equation}
Equivalently, writing $\mathbf e_{a_r}\in\{0,1\}^V$ for the one-hot representation of $a_r$, GCG computes
\[
d_r
\;=\;
\nabla_{\mathbf e_{a_r}}\mathcal L(a)
\in\mathbb R^V,
\]
and ranks candidate replacements by the most negative coordinates of this gradient. For each suffix position $r$, it forms a candidate set
\begin{equation}
\mathcal C_r
\;=\;
\operatorname{TopK}_{v\in[V]}
\left(
-\left(E_v - E_{a_r}\right)^\top
\nabla_{u_r}\mathcal L(a)
\right),
\label{eq:gcg-candidate-set}
\end{equation}
or, equivalently, the top-$k$ entries of $-d_r$ up to the constant corresponding to the current token. GCG then samples a batch of candidate substitutions $(r,v)$ with $v\in\mathcal C_r$, evaluates the true loss \eqref{eq:gcg-loss} for each substituted suffix using forward passes, and commits the replacement with the smallest exact loss. The gradient is therefore used only to construct a promising discrete candidate set; the final greedy step is still selected by evaluating the actual autoregressive objective.

\paragraph{GCG-GPR.}
Our token-level method keeps the GCG loss and greedy candidate evaluation unchanged, but replaces the backward-computed gradient in \eqref{eq:gcg-candidate-set} with a predicted gradient. Specifically, from the same teacher-forced forward pass used to score the target continuation, we extract an intermediate hidden representation $h_\theta^{(\ell_0)}(q,a,t)$ and use our learned gradient predictor to estimate the suffix-embedding loss gradient:
\begin{equation}
\widehat G_\theta(q,a,t)
\;=\;
A\,\widetilde h_\theta^{(\ell_0)}(q,a,t) + b
\;\approx\;
\mathcal N_{\rm tok}\!\left(\nabla_{u_{1:\ell}}\mathcal L_\theta(q,a,t)\right),
\label{eq:gpr-gcg-gradient}
\end{equation}
where $\widetilde h_\theta^{(\ell_0)}$ denotes the standardized hidden-state feature vector and
$\widehat G_\theta(q,a,t)\in\mathbb R^{\ell\times d_e}$ is reshaped into per-position predictions
$\widehat g_r\in\mathbb R^{d_e}$. Here $\mathcal N_{\rm tok}$ normalizes each nonzero token-gradient row to unit Euclidean norm and maps zero rows to zero. This preserves the exact within-position candidate ranking under positive rescaling. We then replace the exact GCG candidate score with

\begin{equation}
\widehat\Delta_{r,v}
\;=\;
\left(E_v - E_{a_r}\right)^\top \widehat g_r,
\qquad
\widehat{\mathcal C}_r
\;=\;
\operatorname{TopK}_{v\in[V]}\left(-\widehat\Delta_{r,v}\right).
\label{eq:gpr-gcg-candidates}
\end{equation}
After this substitution, the rest of GCG is unchanged: we sample candidate token replacements from
$\{\widehat{\mathcal C}_r\}_{r=1}^{\ell}$, evaluate their exact losses by forward pass, and accept the best replacement. This makes GCG-GPR a drop-in acceleration of GCG's candidate-generation step. Standard GCG requires a backward pass at every iteration to obtain $\nabla_{u_r}\mathcal L(a)$ or $\nabla_{\mathbf e_{a_r}}\mathcal L(a)$; GCG-GPR instead obtains an approximate gradient from a forward-pass representation and a lightweight linear map. Importantly, because the final token replacement is still chosen using the true loss, the method does not require the predicted gradient to be exact. It only needs to rank sufficiently many useful replacements into the top-$k$ candidate sets.

\subsection{Experimental hyperparameters}\label{app:token-attack-hparams}
The following settings specify the token-level experiments reported in Section~\ref{sec-attack-token}.

\begin{itemize}
    \item \textbf{Data and splits.} For HarmBench, predictor fitting uses the validation split, which contains 80 prompts, and evaluation uses the 320 held-out test prompts. For JailbreakBench, predictor fitting uses the 300-example judge-comparison split; this split is used only for judge benchmarking and is not used as the test-time jailbreak set. JailbreakBench evaluation uses the 100-prompt harmful test set. Prompt sampling uses seed 42.
    \item \textbf{GCG attack settings.} We initialize each attack with the dummy suffix used by the original GCG setup, namely ten repetitions of \texttt{! }. We run GCG for 25 optimization steps, use the top 64 candidate tokens per suffix position, and evaluate candidate substitutions with the default batch size from the GCG implementation. The target continuation is \texttt{Sure, here is}. We do not apply additional token filtering or allowed-token constraints.
    \item \textbf{Predictor-training augmentation.} On HarmBench, each validation prompt is first paired with the same dummy suffix used to initialize the exploit, namely ten repetitions of \texttt{! } tokens. We then run GCG for 25 steps from this initial suffix and record the initial suffix state together with each subsequent GCG iterate, giving 26 suffix states per prompt and $80\times26=2{,}080$ trajectory examples. To further enlarge the predictor-training set, for each trajectory example we generate seven additional variants by randomly changing one suffix token, while also keeping the original trajectory example. This gives $8\times2{,}080=16{,}640$ HarmBench predictor-training examples.
    \item \textbf{GCG-GPR predictor settings.} We use hidden layer $\ell_0=18$ for gradient prediction. The predictor normalizes each true suffix gradient to unit norm per token, flattens hidden states and gradients across token positions, standardizes hidden-state features coordinatewise, appends a bias feature, and fits a full-rank weighted linear ridge regression. All stored examples use unit sample weight. The ridge parameter is 100. The augmentation procedure is GCG-trajectory augmentation plus random one-token suffix perturbations as described above.
    \item \textbf{Blind random baseline.}
    For each evaluation prompt, we draw one suffix by independently sampling every token in the GCG suffix region uniformly from the tokenizer's non-special vocabulary supported by the target model. We use this suffix directly, with no gradient computation, target-loss evaluation, candidate selection, or iterative search. We report ASR only; perturbation and successful-attack throughput are left undefined.

    \item \textbf{Evaluation and timing.} We use the HarmBench custom judge for HarmBench and JailbreakJudge for JailbreakBench; JailbreakJudge is Llama-3-70B-Instruct quantized and served with Ollama. Attacked-model responses use default generation parameters. ASR is the number of successful attacks divided by the total number of evaluation prompts. We use the default timing boundary from the attack implementation, without an additional timing exclusion. Experiments run on A100 GPUs using bfloat16 precision. Confidence intervals use bootstrap resampling with $B=10{,}000$.
\end{itemize}

\section{Experimental details - embedding attacks}
\subsection{Hyperparameters}
\label{sec-hypers-emb}
\paragraph{Hyperparameter Selection} All hyperparameters are selected on a disjoint dataset split, not used for testing. Tuning proceeds greedily through three phases, each fixing the winners of the previous one and selecting by GPR successful-attack throughput (ASR$\times|\mathcal{D}|/$wall-time). \textbf{Phase 1} (predictor quality) sweeps the joint grid $\text{target layer}\times\text{aug.\ steps}\times\text{aug.\ weight decay}\in{9,12,15,18,21,24}\times{0,5,10,15,20,30}\times{0.7,0.8,0.9,0.95,1.0}$ at fixed $\epsilon{=}0.005$, pruning the redundant decay axis when aug.\ steps$=0$. \textbf{Phase 2} (perturbation budget) sweeps $\epsilon\in{0.001,0.002,0.003,0.005,0.007,0.01,0.015,0.02}$ for $\ell_\infty$ methods and $\epsilon\in{0.25,0.5,0.75,1,1.5,2,3,5}$ for $\ell_2$ methods, choosing the best $\epsilon$ per norm. \textbf{Phase 3} (method-specific HPs) then tunes (a) the PGD inner step size over ${\epsilon/N,2\epsilon/N,10^{-3},2{\times}10^{-3}}$ with $N{=}5$, independently for vanilla PGD and PGD-GPR; (b) the RS-FGSM step multiplier $\alpha/\epsilon\in{0.5,0.75,1,1.25,1.5,2,2.5}$; and (c) the PGD step-count Pareto frontier $N\in{1,2,3,5,7,10,15,20}$.

\paragraph{Hyperparameter Sensitivity.}
Figure~\ref{fig:eps_sensitivity} ablates the perturbation budget~$\epsilon$ and Figure~\ref{fig:weight_decay_ablation} ablates the augmentation weight decay used during predictor training; both knobs trade off ASR against gradient-prediction quality.

\paragraph{Mixed-gradient evaluation protocol.}
\label{sec:embedding-mixed-protocol}
Figure~\ref{fig:alpha_sweep} holds predictor-fitting and attack settings fixed
across mixing weights. HarmBench uses the same non-trivial subsets as
Table~\ref{tab:main}: \AlphaHBQwenTwoN{} prompts for Qwen2.5-3B and
\AlphaHBLlamaTwoN{} for Llama-2-7B, excluding prompts for which the
unperturbed model already succeeds. JailbreakBench uses the same held-out
prompts as Table~\ref{tab:main-jbb}. JailbreakBench evaluations use a shared
Llama-3-70B-Instruct Q4\_K\_M judge with fixed decoding settings. Each unique
behavior--response pair receives one cached verdict, reused across the table
and mixed-gradient curves. Saved responses and perturbation-generation
timings are retained, and statistics are recomputed using the corresponding
prompt subsets and shared verdicts. Both curve endpoints and their ASR
confidence intervals match the corresponding table entries. Bootstrap
intervals quantify prompt-resampling uncertainty with judge verdicts held
fixed.

\begin{figure}[t]
    \centering
    \begin{subfigure}{0.49\linewidth}
        \includegraphics[width=\linewidth]{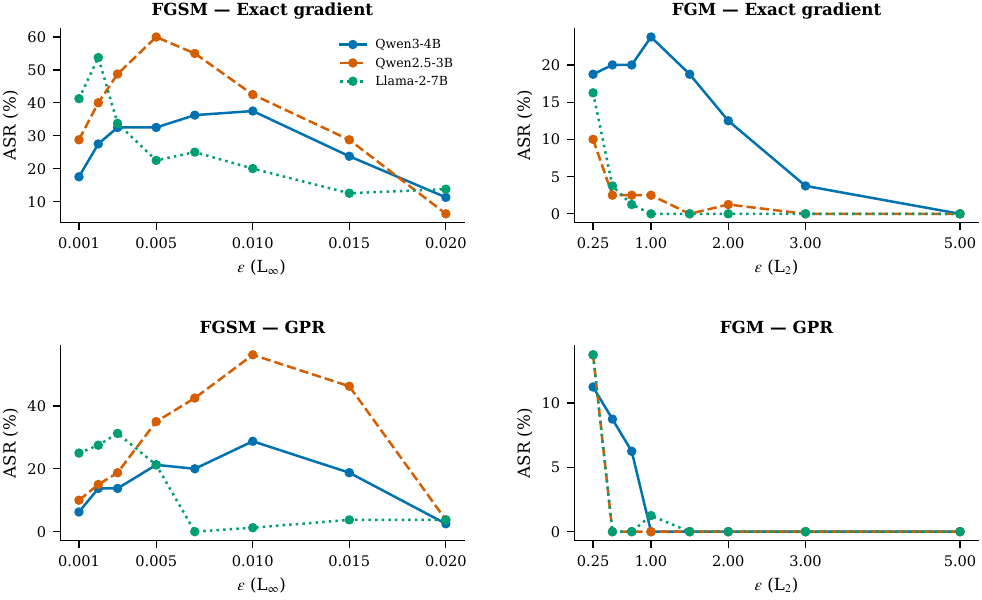}
        \caption{ASR vs.\ $\epsilon$ for exact and GPR variants.}
        \label{fig:eps_sensitivity}
    \end{subfigure}\hfill
    \begin{subfigure}{0.49\linewidth}
        \includegraphics[width=\linewidth]{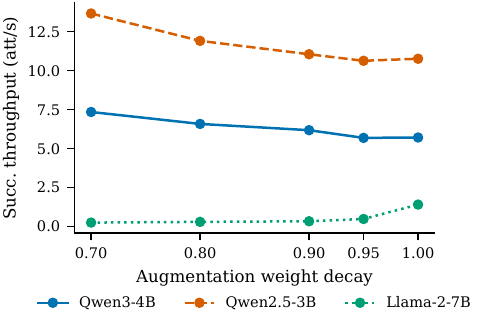}
        \caption{Augmentation weight-decay ablation.}
        \label{fig:weight_decay_ablation}
    \end{subfigure}
    \caption{Sensitivity to perturbation budget $\epsilon$ and predictor-training augmentation weight decay.}
\end{figure}

\paragraph{Layer choice.}
The early-exit layer~$\ell$ trades depth (richer representation) against speed. Figure~\ref{fig:cosine_asr_layer} shows predictor cosine similarity and resulting ASR as a function of~$\ell$, and Figure~\ref{fig:speedup_layer} shows the corresponding speedup over full backpropagation. Cosine similarity is not a perfect proxy for ASR (Figure~\ref{fig:cosine_llama2}), motivating hyperparameter selection by successful-attack throughput.

\begin{figure}[t]
    \centering
    \begin{subfigure}{0.49\linewidth}
        \includegraphics[width=\linewidth]{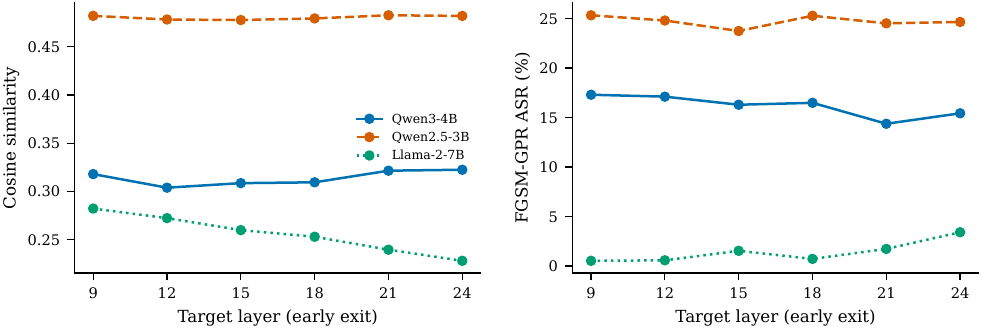}
        \caption{Predictor cosine and ASR vs.\ target layer~$\ell$.}
        \label{fig:cosine_asr_layer}
    \end{subfigure}\hfill
    \begin{subfigure}{0.49\linewidth}
        \includegraphics[width=\linewidth]{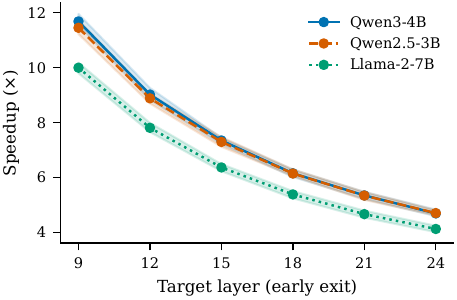}
        \caption{Speedup over full backprop vs.\ target layer~$\ell$.}
        \label{fig:speedup_layer}
    \end{subfigure}
    \caption{Effect of the early-exit layer~$\ell$.}
\end{figure}

\begin{figure}[t]
    \centering
    \includegraphics[width=0.6\linewidth]{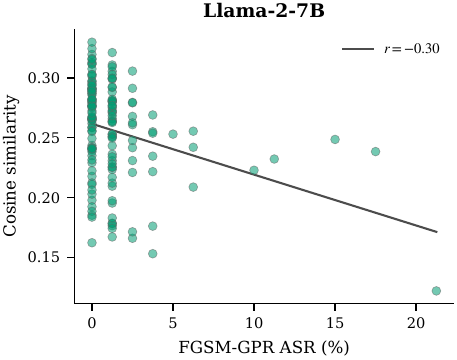}
    \caption{Cosine similarity of predicted gradients vs.\ resulting ASR on Llama-2-7B. The relationship is noisy, showing that cosine isn't the best choice for tuning hyperparameters.}
    \label{fig:cosine_llama2}
\end{figure}

\subsection{Embedding-Level JailbreakBench Results}
\label{sec-emb-jbb}
See Table \ref{tab:main-jbb} for the JailbreakBench results.
\begin{table}[h]
    \centering \footnotesize 
    \caption{{Attack success rate (ASR), perturbation throughput (Pert., examples/s), and success throughput (Succ.) on the non-trivial subset of JailbreakBench held-out prompts (Qwen3-4B: \JBBQwenThreeN{}, Qwen2.5-3B: \JBBQwenTwoN{}, Llama-2-7B: \JBBLlamaTwoN{}). Saved responses are rejudged with shared per-pair verdicts (Appendix~\ref{sec:embedding-mixed-protocol}). 95\% bootstrap CIs from per-prompt resampling ($B=\AlphaBootstrapB$).}}
    \label{tab:main-jbb}
    \setlength{\tabcolsep}{2pt}
    \resizebox{\linewidth}{!}{
\begin{tabular}{@{}l rrr rrr rrr@{}}
\toprule
 & \multicolumn{3}{c}{Qwen3-4B} & \multicolumn{3}{c}{Qwen2.5-3B} & \multicolumn{3}{c}{Llama-2-7B} \\
\cmidrule(lr){2-4} \cmidrule(lr){5-7} \cmidrule(lr){8-10}
Method & ASR\,(\%) & Pert.\,(ex/s) & Succ.\,(att/s) & ASR\,(\%) & Pert.\,(ex/s) & Succ.\,(att/s) & ASR\,(\%) & Pert.\,(ex/s) & Succ.\,(att/s) \\
\midrule
\multicolumn{10}{c}{Exact gradient} \\
\cmidrule{1-10}
FGSM & {$6.0_{-6.0}^{+8.0}$} & $6.0_{-0.0}^{+0.0}$ & {$\mathbf{0.36_{-0.36}^{+0.48}}$} & $22.0_{-10.0}^{+12.0}$ & $7.5_{-0.1}^{+0.1}$ & $1.65_{-0.75}^{+0.89}$ & $6.0_{-6.0}^{+8.0}$ & $8.8_{-0.1}^{+0.1}$ & {$\mathbf{0.53_{-0.53}^{+0.69}}$} \\
FGM & $0.0_{-0.0}^{+0.0}$ & $5.6_{-0.2}^{+0.1}$ & {$\mathbf{0.00_{-0.00}^{+0.00}}$} & $8.0_{-6.0}^{+8.0}$ & $6.9_{-0.4}^{+0.2}$ & $0.55_{-0.42}^{+0.58}$ & $0.0_{-0.0}^{+0.0}$ & $8.9_{-0.1}^{+0.1}$ & {$\mathbf{0.00_{-0.00}^{+0.00}}$} \\
RS-FGSM & $2.0_{-2.0}^{+4.0}$ & $5.9_{-0.1}^{+0.1}$ & {$\mathbf{0.12_{-0.12}^{+0.24}}$} & {$22.0_{-10.0}^{+12.0}$} & $7.4_{-0.1}^{+0.1}$ & {$1.63_{-0.75}^{+0.89}$} & $12.0_{-8.0}^{+10.0}$ & $8.5_{-0.1}^{+0.1}$ & {$\mathbf{1.03_{-0.68}^{+0.85}}$} \\
PGD-2 & $4.0_{-4.0}^{+6.0}$ & $3.0_{-0.1}^{+0.0}$ & {$\mathbf{0.12_{-0.12}^{+0.18}}$} & $10.0_{-8.0}^{+8.0}$ & $3.7_{-0.0}^{+0.0}$ & {$\mathbf{0.37_{-0.30}^{+0.30}}$} & $14.0_{-8.0}^{+10.0}$ & $4.3_{-0.0}^{+0.0}$ & {$\mathbf{0.60_{-0.35}^{+0.43}}$} \\
PGD-5 & $6.0_{-6.0}^{+8.0}$ & $1.2_{-0.0}^{+0.0}$ & {$\mathbf{0.07_{-0.07}^{+0.10}}$} & $40.0_{-14.0}^{+14.0}$ & $1.5_{-0.0}^{+0.0}$ & $0.60_{-0.21}^{+0.21}$ & $24.0_{-12.0}^{+12.0}$ & $1.7_{-0.0}^{+0.0}$ & {$\mathbf{0.42_{-0.21}^{+0.21}}$} \\
PGD-10 & $30.0_{-12.0}^{+14.0}$ & $0.6_{-0.0}^{+0.0}$ & {$\mathbf{0.18_{-0.07}^{+0.08}}$} & $60.0_{-14.0}^{+14.0}$ & $0.8_{-0.0}^{+0.0}$ & $0.45_{-0.11}^{+0.10}$ & $26.0_{-12.0}^{+12.0}$ & $0.9_{-0.0}^{+0.0}$ & {$\mathbf{0.23_{-0.10}^{+0.11}}$} \\
\addlinespace
\multicolumn{10}{c}{GPR (predicted)} \\
\cmidrule{1-10}
FGSM & $0.0_{-0.0}^{+0.0}$ & $59.4_{-0.7}^{+0.6}$ & $0.00_{-0.00}^{+0.00}$ & $4.0_{-4.0}^{+6.0}$ & $94.6_{-2.1}^{+1.9}$ & {$\mathbf{3.78_{-3.78}^{+5.75}}$} & $0.0_{-0.0}^{+0.0}$ & $70.2_{-0.8}^{+0.8}$ & $0.00_{-0.00}^{+0.00}$ \\
FGM & $0.0_{-0.0}^{+0.0}$ & $55.7_{-1.1}^{+0.9}$ & $0.00_{-0.00}^{+0.00}$ & $22.0_{-12.0}^{+12.0}$ & $64.6_{-1.1}^{+1.0}$ & {$\mathbf{14.21_{-7.60}^{+7.70}}$} & $0.0_{-0.0}^{+0.0}$ & $46.1_{-0.5}^{+0.5}$ & $0.00_{-0.00}^{+0.00}$ \\
RS-FGSM & $0.0_{-0.0}^{+0.0}$ & $59.0_{-0.9}^{+0.7}$ & $0.00_{-0.00}^{+0.00}$ & $8.0_{-6.0}^{+8.0}$ & $94.8_{-1.6}^{+1.5}$ & {$\mathbf{7.58_{-5.70}^{+7.65}}$} & $0.0_{-0.0}^{+0.0}$ & $70.0_{-1.1}^{+0.9}$ & $0.00_{-0.00}^{+0.00}$ \\
PGD-2 & $0.0_{-0.0}^{+0.0}$ & $31.2_{-0.4}^{+0.3}$ & $0.00_{-0.00}^{+0.00}$ & $0.0_{-0.0}^{+0.0}$ & $51.6_{-0.7}^{+0.6}$ & $0.00_{-0.00}^{+0.00}$ & $0.0_{-0.0}^{+0.0}$ & $37.6_{-0.3}^{+0.3}$ & $0.00_{-0.00}^{+0.00}$ \\
PGD-5 & $0.0_{-0.0}^{+0.0}$ & $13.0_{-0.1}^{+0.1}$ & $0.00_{-0.00}^{+0.00}$ & $4.0_{-4.0}^{+6.0}$ & $22.3_{-0.2}^{+0.2}$ & {$\mathbf{0.89_{-0.89}^{+1.35}}$} & $0.0_{-0.0}^{+0.0}$ & $15.7_{-0.1}^{+0.1}$ & $0.00_{-0.00}^{+0.00}$ \\
PGD-10 & $0.0_{-0.0}^{+0.0}$ & $6.6_{-0.0}^{+0.0}$ & $0.00_{-0.00}^{+0.00}$ & $16.0_{-10.0}^{+10.0}$ & $11.3_{-0.1}^{+0.1}$ & {$\mathbf{1.81_{-1.13}^{+1.15}}$} & $0.0_{-0.0}^{+0.0}$ & $8.0_{-0.1}^{+0.1}$ & $0.00_{-0.00}^{+0.00}$ \\
\bottomrule
\end{tabular}
}
\end{table}

\subsection{Cross-benchmark transfer and predictor training cost}
\label{sec:embedding-transfer}

We test whether the continuous-gradient predictor transfers beyond its
training distribution in a controlled FGM experiment on Qwen2.5-3B.  The
JailbreakBench (JBB) harmful set is shuffled with seed~0: the first 50
prompt--target pairs are used for training and the remaining 50 for
evaluation.  The HarmBench (HB) predictor is trained on its 80 validation
prompts and evaluated on the 294 non-trivial Qwen2.5 test prompts.  Within each
evaluation set, the attack configuration is fixed (layer~9, 20 augmentation
steps, augmentation weight decay~1, and $\ell_2$ budget $\epsilon=0.25$); only
the predictor-training benchmark changes.  Timing uses one explicit untimed
warm-up step before all measured examples.  Confidence intervals are computed
by 10,000 per-prompt bootstrap resamples.

\begin{table}[t]
    \centering
    \footnotesize

\resizebox{\linewidth}{!}{%
\begin{tabular}{@{}llrrrr@{}}
\toprule
Evaluation & Predictor training & ASR (\%) & Succ. (att/s) & Cosine & Sign agreement (\%) \\
\midrule
JailbreakBench ($n=50$) & Exact gradient & $8.0_{-6.0}^{+8.0}$ & $0.58_{-0.44}^{+0.60}$ & -- & -- \\
 & JBB (in-domain) & $22.0_{-12.0}^{+12.0}$ & $16.30_{-8.63}^{+8.97}$ & $0.43_{-0.02}^{+0.02}$ & $64.5_{-0.5}^{+0.5}$ \\
 & HB (transferred) & $16.0_{-10.0}^{+10.0}$ & $12.18_{-7.59}^{+7.70}$ & $0.37_{-0.02}^{+0.02}$ & $61.3_{-0.5}^{+0.5}$ \\
\addlinespace
HarmBench ($n=294$) & Exact gradient & $6.1_{-2.7}^{+3.1}$ & $0.45_{-0.20}^{+0.22}$ & -- & -- \\
 & HB (in-domain) & $12.6_{-3.7}^{+3.7}$ & $9.61_{-2.84}^{+2.88}$ & $0.39_{-0.01}^{+0.01}$ & $61.0_{-0.2}^{+0.2}$ \\
 & JBB (transferred) & $20.7_{-4.4}^{+4.8}$ & $15.79_{-3.43}^{+3.67}$ & $0.33_{-0.01}^{+0.01}$ & $59.4_{-0.2}^{+0.2}$ \\
\bottomrule
\end{tabular}%
}

\vspace{0.75em}

\begin{tabular}{@{}llrr@{}}
\toprule
Evaluation & Dataset & Predictor Training (s) & Break-even examples \\
\midrule
JailbreakBench & JBB (in-domain) & 159.2 & 1280 \\
JailbreakBench & HB (transferred) & 259.7 & 2082 \\
HarmBench & HB (in-domain) & 259.7 & 2113 \\
HarmBench & JBB (transferred) & 159.2 & 1296 \\
\bottomrule
\end{tabular}

    \caption{{Controlled cross-benchmark transfer for continuous FGM (top) and
    its one-time predictor-training cost (bottom).  Successful-attack
    throughput is the number of successful attacks divided by total
    perturbation time.  Cosine similarity and sign agreement are reported for predicted-gradient
    conditions, comparing each prediction with its corresponding exact
    gradient.  The
    break-even count is
    $T_{\mathrm{cal}}/(\bar t_{\mathrm{exact}}-\bar t_{\mathrm{GPR}})$.
    Intervals are 95\% bootstrap CIs ($B=10{,}000$).}}
    \label{tab:embedding-transfer}
\end{table}

Transfer is asymmetric in this setting.  On JBB, in-domain training gives
higher ASR than the HB-trained predictor; on HB, the JBB-trained predictor gives
higher ASR than the in-domain predictor despite lower cosine similarity and
sign agreement.  Thus these predictor-quality diagnostics should be interpreted alongside
direct measurements of attack performance.  Both transferred predictors
nevertheless achieve substantially higher successful-attack throughput than
exact FGM in these evaluations.  This experiment probes transfer across the
joint prompt--target distributions of two benchmarks for one fixed model and
attack configuration. Broader evaluations would be needed to assess model-,
objective-, or target-agnostic gradient prediction.  The table also makes explicit
that GPR is intended for workloads large enough to amortize its one-time fitting
cost.

\subsection{Further Details on Throughput}
\label{sec-exp-throughput}
\paragraph{Throughput comparison.}
Figure~\ref{fig:throughput_comparison} contrasts per-method successful-attack throughput between exact and GPR variants on a linear scale.

\begin{figure}[t]
    \centering
    \includegraphics[width=\linewidth]{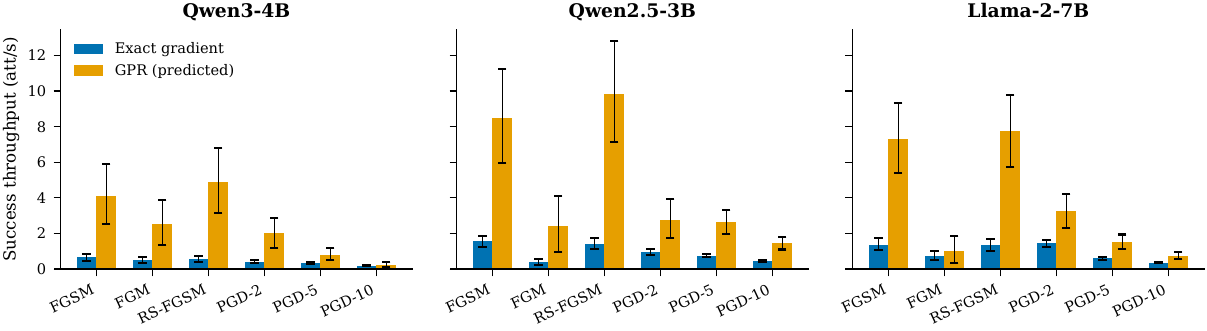}
    \caption{Per-method successful-attack throughput on HarmBench.}
    \label{fig:throughput_comparison}
\end{figure}



\section{Additional background on GP derivatives}\label{app:gp-derivatives}
This appendix records a standard GP derivative-conditioning formula used as motivation in the main text.
Let $f:\R^d\to\R^m$ be a (vector-valued) Gaussian process with mean $m(\cdot)$ and kernel
$K(\cdot,\cdot)$ so that $K(x,x')=\Cov(f(x),f(x'))$ \citep{williams2006gaussian}.
Assume $m$ is continuously differentiable and $K$ is twice continuously differentiable (entrywise).
Then the partial-derivative process $g_j(x)\coloneqq \partial f(x)/\partial x_j\in\R^m$ exists (in a mean-square sense)
and $(f,g_j)$ is jointly Gaussian \citep[Chapter 9.4]{williams2006gaussian}.
Given a finite input set $X=\{x_1,\dots,x_n\}$ and noiseless observations $y=f(X)$, assume $K_{ff}(X,X)$ is invertible. Conditioning the joint Gaussian then yields
the posterior mean of the derivative at $x_\star$:
\begin{equation}
\E\!\left[g_j(x_\star)\mid f(X)=y\right]
=
m_{g_j}(x_\star)
+
K_{g_j f}(x_\star,X)\, K_{ff}(X,X)^{-1}\,\big(y - m_f(X)\big),
\label{eq-affine}
\end{equation}
with $m_f(X)$ the stacked mean at $X$, $m_{g_j}(x_\star)=\partial m(x_\star)/\partial x_j$, and
$\big[K_{g_j f}(x_\star,X)\big]_i = \frac{\partial}{\partial x_{\star,j}}K(x_\star,x_i)$ (blockwise).

\paragraph{Exact conditional means in the NTK limit.}
Fix regression data $\mathcal D=(X,y)$, with all output coordinates stacked in $y$.
Let $\mathcal K$ be the deterministic limiting NTK and let the centered initialization
$f_0$ have GP covariance $K_0$. Assume $\mathcal K_{XX}:=\mathcal K(X,X)$ is invertible.
Squared-loss gradient flow in the NTK limit gives \citep[Section~5]{jacot2018neural}
\[
f_\tau(x)=f_0(x)+Q_\tau(x)\bigl(y-f_0(X)\bigr),\qquad
Q_\tau(x)=\mathcal K(x,X)\mathcal K_{XX}^{-1}
\bigl(I-e^{-\tau\mathcal K_{XX}}\bigr).
\]
The infinite-time limit replaces the last factor by $I$.
For fixed data, this is an affine transformation of $f_0$, so $f_\tau$ is a GP over
random initializations. Its mean is $Q_\tau(x)y$; an individual trained function
also retains the centered random term $f_0(x)-Q_\tau(x)f_0(X)$.
Its covariance is induced by $K_0$ and $Q_\tau$, and is generally not $\mathcal K$.
Assume $K_0$ and $\mathcal K$ are twice continuously differentiable and take
$H_\tau(x)=f_\tau(x)$ and $G_\tau(x)=\nabla_x(c^\top f_\tau(x))$ for fixed $c$.
Their joint Gaussianity follows from mean-square differentiation of this limiting
process. Writing their pointwise means and centered covariance blocks as
$\mu_{H,\tau}(x)$, $\mu_{G,\tau}(x)$, $C_{HH,\tau}(x)$, and $C_{GH,\tau}(x)$ gives
\[
\E[G_\tau(x)\mid H_\tau(x)]=A_\tau(x)H_\tau(x)+b_\tau(x),
\]
\[
A_\tau(x)=C_{GH,\tau}(x)C_{HH,\tau}(x)^{-1},\qquad
b_\tau(x)=\mu_{G,\tau}(x)-A_\tau(x)\mu_{H,\tau}(x),
\]
where $C_{HH,\tau}(x)$ is invertible; a Moore--Penrose inverse gives the
corresponding conditional mean when it is singular.
This is an exact pointwise conditional-mean result. Input-independent coefficients
follow under the additional moment conditions of Proposition~\ref{prop:affine-predictor}.
Stationarity of the initialization or NTK kernel alone does not establish these
conditions after training on fixed data. Expectations here are over initialization,
not over prompts for a single fixed trained network.

\paragraph{Exact recovery in a fixed-feature NTK model.}
Separately, let $\phi(x)=(\cos x,\sin x)^\top$ and let the final regression outputs be
$f_B(x)=B\phi(x)\in\R^2$, with only $B\in\R^{2\times2}$ trained.
Its matrix NTK is $\mathcal K(x,x')=\cos(x-x')I_2$, independent of $B$;
thus squared-loss training is exactly fixed-kernel training.
For any fitted $B$ that is invertible, set $H(x)=f_B(x)$ and
$J=\left[\begin{smallmatrix}0&-1\\1&0\end{smallmatrix}\right]$. Then
\[
\frac{d f_B(x)}{dx}=BJ\phi(x)=BJB^{-1}H(x).
\]
Hence each output's input derivative is recovered exactly by a model-specific
linear map of the final outputs. The identity follows from $\phi'=J\phi$ and
holds after training whenever $B$ remains invertible. This is a restricted
exact-recovery example, not a consequence of Gaussianity alone or a guarantee
for arbitrary NTK architectures or finite-sample ridge estimates.

\section{Proofs}\label{app:proofs}

\begin{proof}[Proof of Lemma~\ref{lem:deep-sincos-stationary}]
Write $\eta=\{(\omega_j,u_j,v_j)\}_{j=1}^m$ for the embedding
parameters. Fix a translation $c\in\R^d$ and define
\[
Q(\alpha)=
\begin{bmatrix}
\cos\alpha & \sin\alpha\\
-\sin\alpha & \cos\alpha
\end{bmatrix},
\qquad
\begin{bmatrix}u_j'\\v_j'\end{bmatrix}
=
Q(\omega_j^\top c)
\begin{bmatrix}u_j\\v_j\end{bmatrix}.
\]
Let $\eta'=\{(\omega_j,u_j',v_j')\}_{j=1}^m$.
The identity $Q(\alpha+\delta)=Q(\alpha)Q(\delta)$ gives
\[
\begin{aligned}
e_j(x+c;\eta)
&=
Q(\omega_j^\top x+\omega_j^\top c)
\begin{bmatrix}u_j\\v_j\end{bmatrix}\\
&=
Q(\omega_j^\top x)
\begin{bmatrix}u_j'\\v_j'\end{bmatrix}
=e_j(x;\eta').
\end{aligned}
\]
Thus $e(x+c;\eta)=e(x;\eta')$ simultaneously for all $x$.
Conditional on the frequencies, the transformation applies
deterministic orthogonal rotations to independent standard
Gaussian pairs. Their joint law is unchanged. Hence
$\eta'\stackrel{d}{=}\eta$, and independence from $\Theta$ gives
\[
(\eta',\Theta)\stackrel{d}{=}(\eta,\Theta).
\]

Because the network depends on $x$ only through $e(x)$,
\[
H(x+c;\eta,\Theta)=H(x;\eta',\Theta),
\qquad
Y(x+c;\eta,\Theta)=Y(x;\eta',\Theta).
\]
These identities use the same transformed parameters for every
input. Equality of the parameter
laws therefore proves joint stationarity of $(H,Y)$.

Finite second moments give constant means and, for
$U,V\in\{H,Y\}$,
\[
\Cov(U(x),V(x'))
=
\Cov(U(x-x'),V(0)).
\]
Thus all auto- and cross-covariances depend only on $x-x'$.
Likewise,
\[
k(x+c,x'+c)
=
\E[Y(x+c)Y(x'+c)]
=
\E[Y(x)Y(x')]
=
k(x,x').
\]

Finally, differentiate the pathwise identity
$Y(x+c;\eta,\Theta)=Y(x;\eta',\Theta)$ with respect to $x$.
Since $\eta'$ depends on $c$ but not on $x$, almost-sure
differentiability yields
\[
G(x+c;\eta,\Theta)=G(x;\eta',\Theta).
\]
Equality of the parameter laws and integrability now imply
\[
\mu_G(x+c)
=
\E[G(x+c;\eta,\Theta)]
=
\E[G(x;\eta',\Theta)]
=
\mu_G(x).
\]
Thus $\mu_G$ is constant. This expectation is the unconditional
prior mean over $(\eta,\Theta)$, not the conditional prediction
$\E[G(x)\mid H(x)]$. No interchange of differentiation and
expectation is needed.
\end{proof}

\begin{proof}[Proof of Lemma~\ref{lem:stationary-implies-constant-sigma}]
Fix a translation $c$. The coefficient rotation in the proof of
Lemma~\ref{lem:deep-sincos-stationary} replaces the embedding
parameters $\eta$ by $\eta'$ without changing their joint law
with $\Theta$, and gives
\[
H(x+c;\eta,\Theta)=H(x;\eta',\Theta),
\qquad
Y(x+c;\eta,\Theta)=Y(x;\eta',\Theta).
\]
Since $\eta'$ does not depend on $x$, differentiating the second
identity gives
\[
G(x+c;\eta,\Theta)=G(x;\eta',\Theta).
\]
Consequently,
\[
(H(x+c),G(x+c))\stackrel{d}{=}(H(x),G(x)).
\]
The finite second moments ensure that the covariances exist,
and equality of these joint distributions makes both
$\Sigma_{HH}(x)$ and $\Sigma_{GH}(x)$ translation-invariant,
hence constant.
\end{proof}




\begin{proof}[Proof of Proposition~\ref{prop:affine-predictor}]
Fix $x$. Gaussian conditioning gives
\[
\E[G(x)\mid H(x)]
=
\mu_G+\Sigma_{GH}\Sigma_{HH}^{-1}(H(x)-\mu_H)
=
AH(x)+b.
\]
The means and covariance blocks are input-independent,
so the same $A,b$ work for every $x$.
The conditional mean minimizes squared error among all
measurable predictors based on $H(x)$.

The residual $G(x)-AH(x)-b$ has mean zero and covariance
\[
\Cov(G(x),G(x))
-\Sigma_{GH}\Sigma_{HH}^{-1}\Sigma_{GH}^{\top}.
\]
It vanishes almost surely exactly when this covariance is zero,
which proves the exact-recovery criterion.
\end{proof}

\begin{proof}[Proof of observation \ref{obs-best-affine}]
Write
\[
\mu_H=\E[H],\qquad \mu_G=\E[G],\qquad
\Sigma_{HH}=\Cov(H,H),\qquad
\Sigma_{GH}=\Cov(G,H).
\]
The proposed coefficients and their prediction residual are
\[
A=\Sigma_{GH}\Sigma_{HH}^{-1},
\qquad
b=\mu_G-A\mu_H,
\qquad
\varepsilon=G-AH-b.
\]
All expectations below exist because $H$ and $G$ have finite
second moments.

\emph{Step 1: the residual has mean zero.}
The coefficients $A$ and $b$ are deterministic, so linearity
of expectation gives
\[
\begin{aligned}
\E[\varepsilon]
&=\E[G-AH-b]\\
&=\E[G]-A\E[H]-b\\
&=\mu_G-A\mu_H-(\mu_G-A\mu_H)\\
&=0.
\end{aligned}
\]
Thus the choice of intercept $b$ ensures that the prediction
has the same mean as $G$.

\emph{Step 2: the residual is orthogonal to every coordinate of $H$.}
Substituting the definition of $b$, we can write the residual
using centered variables:
\[
\begin{aligned}
\varepsilon
&=G-AH-(\mu_G-A\mu_H)\\
&=(G-\mu_G)-A(H-\mu_H).
\end{aligned}
\]
Also, $H=(H-\mu_H)+\mu_H$. Therefore,
\[
\begin{aligned}
\E[\varepsilon H^\top]
&=\E[\varepsilon(H-\mu_H)^\top]
  +\E[\varepsilon]\mu_H^\top\\
&=\E[\varepsilon(H-\mu_H)^\top]\\
&=\E\!\left[
  \bigl((G-\mu_G)-A(H-\mu_H)\bigr)(H-\mu_H)^\top
  \right]\\
&=\E[(G-\mu_G)(H-\mu_H)^\top]
  -A\E[(H-\mu_H)(H-\mu_H)^\top]\\
&=\Sigma_{GH}-A\Sigma_{HH}\\
&=\Sigma_{GH}
  -\Sigma_{GH}\Sigma_{HH}^{-1}\Sigma_{HH}\\
&=0.
\end{aligned}
\]
The second equality uses Step 1, and the fifth uses the
definition of covariance. In particular,
$\E[\varepsilon_i H_j]=0$ for every pair of coordinates $i,j$.
This is an orthogonality statement, not a claim that
$\varepsilon$ and $H$ are independent.

\emph{Step 3: compare with any other affine predictor.}
Let $CH+v$ be an arbitrary competing affine predictor, with
deterministic coefficients $C,v$. Define the difference between
the proposed prediction and the competing prediction by
\[
\Delta(H):=(AH+b)-(CH+v)=(A-C)H+(b-v).
\]
The competing predictor's error is then
\[
\begin{aligned}
G-CH-v
&=(G-AH-b)+(AH+b-CH-v)\\
&=\varepsilon+\Delta(H).
\end{aligned}
\]
Expanding its expected squared norm gives
\[
\begin{aligned}
\E\|G-CH-v\|_2^2
&=\E\!\left[
  (\varepsilon+\Delta(H))^\top(\varepsilon+\Delta(H))
  \right]\\
&=\E[\varepsilon^\top\varepsilon]
  +\E[\varepsilon^\top\Delta(H)]
  +\E[\Delta(H)^\top\varepsilon]
  +\E[\Delta(H)^\top\Delta(H)]\\
&=\E\|\varepsilon\|_2^2
  +2\E[\varepsilon^\top\Delta(H)]
  +\E\|\Delta(H)\|_2^2.
\end{aligned}
\]
The two cross terms are equal because they are transposes of
the same real scalar.

To see explicitly why their expectation is zero, write
$G,\varepsilon\in\R^d$ and $H\in\R^{d_h}$. Then
\[
\begin{aligned}
\E[\varepsilon^\top\Delta(H)]
&=\E[\varepsilon^\top(A-C)H]
  +\E[\varepsilon^\top(b-v)]\\
&=\sum_{i=1}^{d}\sum_{j=1}^{d_h}
  (A-C)_{ij}\E[\varepsilon_i H_j]
  +\sum_{i=1}^{d}(b-v)_i\E[\varepsilon_i]\\
&=0.
\end{aligned}
\]
The first sum vanishes by Step 2, and the second by Step 1.

Substituting this into the error expansion yields
\[
\begin{aligned}
\E\|G-CH-v\|_2^2
&=\E\|\varepsilon\|_2^2+\E\|\Delta(H)\|_2^2\\
&=\E\|G-AH-b\|_2^2
  +\E\|(A-C)H+(b-v)\|_2^2\\
&\geq \E\|G-AH-b\|_2^2.
\end{aligned}
\]
The additional term is nonnegative and equals zero when
$C=A$ and $v=b$. Hence $AH+b$ minimizes expected squared
error among all affine predictors. No Gaussianity or
stationarity assumption is used.
\end{proof}

\end{document}